%% file: main_arxiv.tex
\pdfoutput=1
\documentclass[11pt]{article}

\usepackage{acl}

\usepackage{times}
\usepackage{latexsym}
\usepackage[ruled, vlined, noend]{algorithm2e}
\usepackage[T1]{fontenc}
\usepackage{booktabs}
\usepackage{caption}
\usepackage{subcaption}

\usepackage[utf8]{inputenc}

\usepackage{microtype}
\usepackage{graphicx}
\usepackage{inconsolata}
\usepackage{algorithm,algorithmic}
\usepackage{amsthm,amsmath,amssymb}
\newtheorem{theorem}{Theorem}
\newtheorem{proposition}{Proposition}

\newtheorem{lemma}{Lemma}

\title{Ranking Large Language Models without Ground Truth}

\author{\bf
        Amit Dhurandhar$^{1*}$,
        Rahul Nair$^{2*}$,
        Moninder Singh$^{1*}$, \\
        \bf
        Elizabeth Daly$^2$,
        Karthikeyan Natesan Ramamurthy$^1$ \\ \\
        $^1$IBM Research, NY, USA, $^2$IBM Research, Ireland} 
        
\begin{document}
\maketitle
\begingroup\def\thefootnote{*}\footnotetext{Equal contribution}\endgroup

\input{text/abstract}

\input{text/intro}
\input{text/related}

\input{text/methods}
\input{text/experiments}

\input{text/concl}

\section*{Limitations}
\label{sec:limitations}

We tested our proposed method only on English language corpuses. This could have introduced bias in the choices we made in implementations. Further, in the use case with Dialog data, the moral judgements of human annotators are subjective and hence the original rankings of the models may vary based on the annotators. These two limitations also hold for a vast majority of human preference benchmarks.

Our ranking will be subpar if there is high correlation amongst incorrect answers, which can happen in multiple choice with few options as is also seen in the experiments. This is why we recommend using our method for more complex generative tasks where the set of possible responses is large.

In addition, ROUGE is an imperfect metric even though it is widely used. We call upon the community to work on more sophisticated metrics that can be easily adapted based on the context. Finally, if the models that are used in ranking are all subpar in performance, our ranking could be incorrect, but in this case may be even a correct ranking has limited utility.

\section*{Ethics Statement}
Our work has the potential to improve the trustworthiness of LLMs by making larger scale evaluations possible with reduced human effort, since we do not need access to a reference dataset. Lack of substantial evaluations is one of the blockers for trustworthy adoption of LLMs, and it could be eased by our approach. In addition, this could be beneficial when the human labor involved for creating reference data or preferences requires looking at malicious content which could lead to psychological harm for the humans involved. It is however also true, that the ranking provided by our method should only be considered as an initial estimate and human oversight is necessary to ensure that the quality of models deemed good are sufficiently good for the application at hand.

\section*{Acknowledgements}

We thank Kush Varshney for initial discussions as well as encouraging the appropriate collaborations. We thank Djallel Bouneffouf for sharing data and human annotations on the dialog use case.

% Scientific work published at ACL 2023 must comply with the ACL Ethics Policy.\footnote{\url{https://www.aclweb.org/portal/content/acl-code-ethics}} We encourage all authors to include an explicit ethics statement on the broader impact of the work, or other ethical considerations after the conclusion but before the references. The ethics statement will not count toward the page limit (8 pages for long, 4 pages for short papers).

%\section*{Acknowledgements}

%\nocite{Ando2005}
\bibliography{refs}
\bibliographystyle{acl_natbib}
\newpage
\input{text/appendix}

\end{document}

%% file: text/abstract.tex
\begin{abstract}
Evaluation and ranking of large language models (LLMs) has become an important problem with the proliferation of these models and their impact. Evaluation methods either require human responses which are expensive to acquire or use pairs of LLMs to evaluate each other which can be unreliable. In this paper, we provide a novel perspective where, given a dataset of prompts (viz. questions, instructions, etc.) and a set of LLMs, we
rank them without access to any ground truth or reference responses. Inspired by real life where both an expert and a knowledgeable person can identify a novice our main idea is to consider triplets of models, where each one of them evaluates the other two, correctly identifying the worst model in the triplet with high probability. We also analyze our idea and provide sufficient conditions for it to succeed.
Applying this idea repeatedly, we propose two methods to rank LLMs. In experiments on different generative tasks (summarization, multiple-choice, and dialog), our methods reliably recover close to true rankings without reference data. This points to a viable low-resource mechanism for practical use\footnote{Code available at \url{https://huggingface.co/spaces/ibm/llm-rank-themselves}.}.
%On two real summarization tasks %leveraging HELM 
%and on a real dialog use case we show that our proposed approaches recover the rankings accurately even without ground truth information. 
\end{abstract}

%% file: text/intro.tex
\section{Introduction}
\label{sec:intro}
% start with benchmarking in general. Include uncertainty part that is discused below in the bullets. (https://arxiv.org/abs/2401.12794)
Recent advancement in LLM capabilities have resulted in a significant challenge for assessing or measuring these capabilities. To benchmark LLM performance, we need a set of input prompts, outputs from the LLMs for these prompts, and a metric that measures how good the LLM performance is on its own or in comparison with other models. 

\begin{figure}[htbp]
    \centering
    \includegraphics[width=.99\columnwidth]{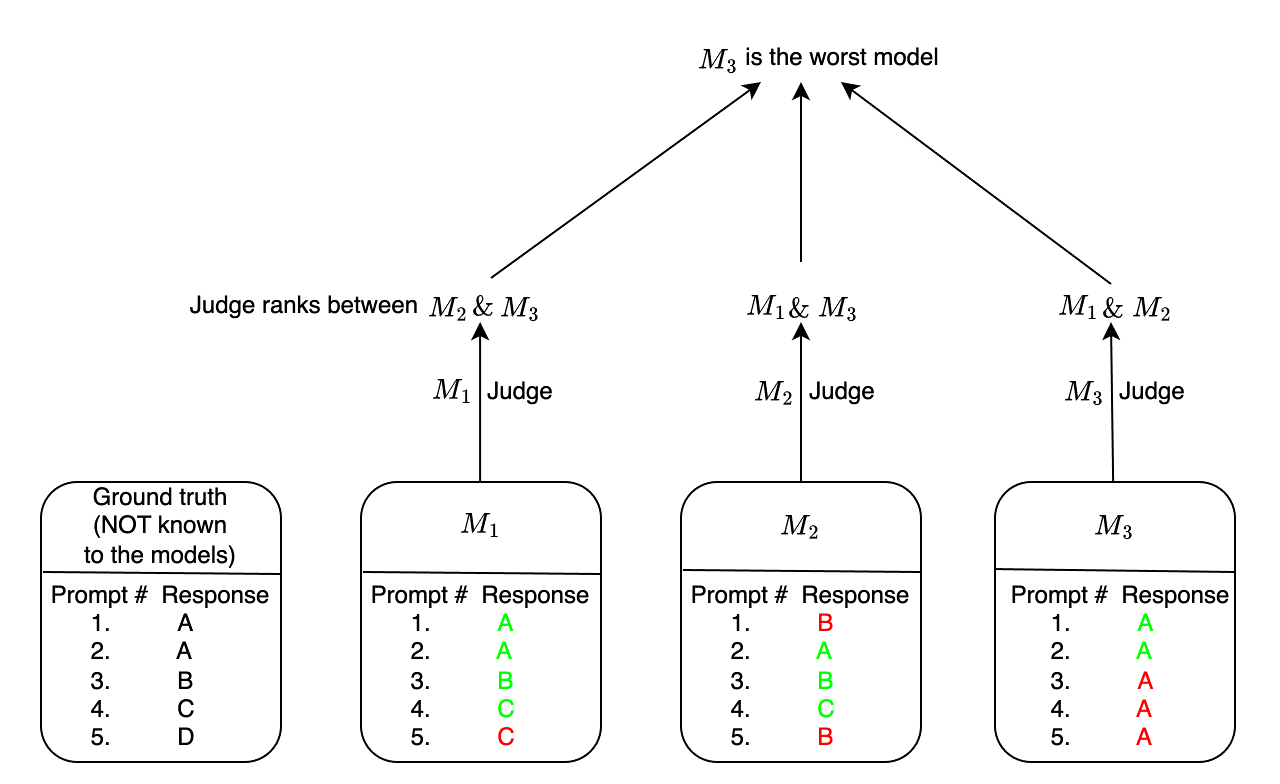}
    \caption{\small{We see the intuition behind the triplet approach. 
    %As can be seen 
    The three models $M_1$, $M_2$ and $M_3$ have accuracies of $80\%$, $60\%$ and $40\%$ respectively based on their responses to five prompts (green are correct responses and red incorrect) when compared with the ground truth which is \emph{unknown} to us. Our triplet approach ranks $M_3$ as the worst model, since it is ranked as such by both $M_1$ (only two answers match) and $M_2$ (only one answer matches). This core idea (with slight variations) can be applied repeatedly to rank an arbitrary number of models as described by the algorithms in Section \ref{sec:meth}.}
    }
    \label{fig:triplet}
    %\vspace{-.25cm}
\end{figure}

%MOVED TO RELATED FOR NOW
%In addition to performance benchmarking, there can be trust and safety benchmarking along different dimensions such as bias, toxicity, social stigma, stereotype, privacy, adversarial robustness, social norms, and machine ethics. There are also comprehensive benchmarks such as HELM \cite{liang2023holistic} that include both performance and safety measures. The benchmarking results are conveniently hosted in locations such as the Open LLM Leaderboard \cite{open-llm-leaderboard, eval-harness}, HELM leaderboard \cite{liang2023holistic}, TrustLLM leaderboard \cite{sun2024trustllm}, and LLM Safety leaderboard \cite{wang2023decodingtrust}.

% Some of the benchmarks that are mostly focused on performance include . The trust and safety benchmarks include . 

%%MOVED TO RELATED FOR NOW
%For tasks with binary labels such as classification, or multiple choice question answering (MCQA), usually the metric used is accuracy. \citep{ye2024benchmarking} discuss the necessity for incorporating uncertainty when measuring an LLM's performance. For generative tasks, the metrics used can be based on comparison of the generated text to a reference text such as BLEU \cite{papineni2002bleu}, ROUGE \cite{lin-2004-rouge}, ROUGE-N \cite{10.3115/1073445.1073465}, METEOR \cite{banerjee2005meteor}, and BERT score \cite{Zhang2020BERTScore}. Libraries such as Unitxt \cite{unitxt} combine these metrics and many others in a customizable manner. Computing accuracy or a related measure on human judgements or preferences is also a very common approach for benchmarking. 

A common strategy in benchmarking is to use datasets that come with both input prompts and reference (``ground truth'') responses in which case automated metrics can assess the model responses by comparing them with reference responses. These benchmarks have been shown to be sensitive to minor perturbations \cite{zhu2023judgelm} and the growing capabilities of LLMs can quickly render a static benchmark out-of-date \cite{laskar2023systematic}.  ``Ground truth'' labels are more readily available in classification and multi-choice question answering but the problem of collecting labels is more challenging in the context of generative tasks.  Additionally, research has also shown benchmarks which target certain metrics can be a poor proxy for assessing an LLM's performance \cite{ethayarajh2020utility}.

%There can be situations where only just input prompts are present in the dataset in which case human judgements need to be collected on the responses from the models. 

% Benchmarks negatives:
% require ground truth labels
% target a specific metric which is not rich
% sensitive to perturbations
% can't keep up with growing capabilities that need to be assessed

For these reasons, the idea of LLM-as-a-judge is starting to be employed in several scenarios in order to serve as a more accurate proxy for human preferences, which cannot be captured easily by simple metrics. However, LLM-as-a-judge may suffer biases  \cite{zhu2023judgelm} and also assume the pre-selected choice of a trusted LLM to serve as the judge. %This paper aims to address this problem by identifying which LLMs could act as judges without prior knowledge.
This paper aims to address the scenarios under which we cannot determine which LLMs can be trusted as judges.

We consider the setup where we have $n$ LLMs and a dataset of prompts, which could be questions like in Q\&A datasets or instructions such as in summarization or in other extractive/generative tasks. 
\textit{Our goal then is to rank these $n$ LLMs in the order of their performance on the chosen task. A priori, we do not assume anything about the quality of the models nor do we assume access to ``ground truth'' or reference responses that can act as the golden standard.}

The core idea of our method, which can be leveraged to design algorithms that rank models in practice, is  illustrated in Figure \ref{fig:triplet}. We consider three models at a time and let each one judge the rest. Based on the cumulative judgements, we decide which is the worst model for that round. 
Note that if we consider only two models, we will not be able to decide the worst model in the round with much trust, since there is no a priori assumption on model quality. The triplet of models idea stems from the real-life intuition that an expert in an area should be able distinguish between a knowledgeable person and a novice. The knowledgeable person should also be able to rank the novice lower than the expert. Hence, both the expert and the knowledgeable person will rank the novice lower than the other and this consensus will help us identify the novice.

This core idea is a part of both the greedy and full ranking methods we propose in this paper (Algorithms \ref{algo:gtr} and \ref{alg:ranking}). Other strategies may also be possible that stem from this idea. Being a judge in Figure \ref{fig:triplet} does not necessarily imply only prompting the judge to choose between two responses, but one could also use NLP metrics (viz. ROUGE, BERT Score, NLI models) to find \textit{closeness} of the responses of other models to the judge's responses, and pick the closest response as the winner. In contrast to our approach, in recommender systems, \textit{items} are recommended to \textit{users} based on some notion of user preference, whereas here the LLMs (equivalently considered as \textit{items}) rank each other without any other intervention or availability of partial ratings.

As a part of this work, we analyze the triplet approach and provide conditions for it to succeed. We also discuss the time complexities of the proposed methods. We show that using summarization, multiple choice question-answering, and dialog as applications, our methods retrieve rankings reliably without any reference data. We exceedingly see the benefit of this approach in summarization and dialog, that is tasks where the responses are long texts, rather than single tokens such as in multiple choice, as other strategies such as ensembling of responses could be employed in the latter case.

%% file: text/related.tex
\section{Related Work}
\label{sec:related}

\textbf{Benchmarking} Benchmarking is heavily relied upon by the community to assess the performance of LLMs. It is viewed as one of the most important problems that need to be addressed with urgency, but it is also recognized that there are no ``one-size-fits-all`` solutions to this problem.  In addition to performance benchmarking, there can be trust and safety benchmarking along different dimensions such as bias, toxicity, social stigma, stereotype, privacy, adversarial robustness, social norms, and machine ethics. There are also comprehensive benchmarks such as HELM \cite{liang2023holistic} that include both performance and safety measures. The benchmarking results are conveniently hosted in locations such as the Open LLM Leaderboard \cite{open-llm-leaderboard, eval-harness}, HELM leaderboard \cite{liang2023holistic}, TrustLLM leaderboard \cite{sun2024trustllm}, and LLM Safety leaderboard \cite{wang2023decodingtrust}. 
%\cite{perlitz2023efficient} - efficient benchmarking for LLMs. Not sure this adds much to the discussion.

For tasks with binary labels such as classification, or multiple choice question answering (MCQA), usually the metric used is accuracy. \citep{ye2024benchmarking} discuss the necessity for incorporating uncertainty when measuring an LLM's performance. For generative tasks, the metrics used can be based on comparison of the generated text to a reference text such as BLEU \cite{papineni2002bleu}, ROUGE \cite{lin-2004-rouge}, ROUGE-N \cite{10.3115/1073445.1073465}, METEOR \cite{banerjee2005meteor}, and BERT score \cite{Zhang2020BERTScore}. Natural Language Inferencing (NLI) benchmarks require human level understanding in order to assess given a premise statement whether it supports (entailment), contradicts or  has no relationship (neutral) in relation to a hypothesis statement \cite{2013Dagan,bowman2015large}. Libraries such as Unitxt \cite{unitxt} combine these metrics and many others in a customizable manner. Computing accuracy or a related measure on human judgements or preferences is also a very common approach for benchmarking. 

Human preference of model responses along the dimensions of helpfulness and harmlessness are collected in the HH-RLHF dataset \citep{bai2022training}. Another significant dataset with quality labels and human preferences of model responses is OpenAssistant Conversations \cite{kopf2023openassistant}. Human preference data for model responses to 80 multi-turn questions was collected in the MT-Bench dataset \cite{zheng2023judging}.

\paragraph{Model Selection}
There are benchmarking deployments \textit{in the wild} such as \textit{Chatbot Arena}\footnote{\url{https://chat.lmsys.org/}} \cite{zheng2023judging} where humans provide the input prompts as well as provide their preference based on the output of two models. Chatbot Arena uses the Elo rating system \cite{elo1967proposed}, popularly used for rating Chess players, to rank LLMs. While collecting the input prompts to create a dataset is time-consuming, it is even more challenging to collect human labels, judgements, or preferences, particularly in cases where the deployment of models need to be made domain-specific. The number of applications and use cases may be so varied, and the human expertise need may be so high, that it will be too expensive to collect such data. 

% \cite{ajtai2015sorting} - ranking with noisy comparisons (Only considers the case where outcome of a comparison is arbitrary if outcomes are similar in performance)

\paragraph{LLM-as-a-judge} Trusted LLMs have been used as a proxy for human preferences and feedback \cite{chiang2023can}. Collecting human preferences are costly, however it has been shown metrics associated with benchmarks are not rich enough to capture the many and nuanced aspects that quantify what represents a \textit{good} output a person might prefer. PandaLM \cite{wang2023pandalm} is trained on human-annotations in order to evaluate model performance on criteria difficult to capture in existing benchmarks reflecting human preferences in order to select a model out of several candidates. Other applications include acting as substitute for a traditional measurement of a user defined criteria. Instead of creating a traditional classification or regression model, the practitioner can specify the custom criteria in free text \footnote{\url{https://python.langchain.com/docs/guides/evaluation/string/criteria_eval_chain}}. More generally, LLM-as-a-judge has also been used to rank model performance \cite{chiang2023vicuna}. 

While more scalable than human-evaluations, leveraging an LLM-as-a-judge has some limitations. \citet{zhu2023judgelm} demonstrated that LLM judgments may suffer from positional bias (e.g. prefers the first answer), knowledge bias (e.g. lacks the knowledge required) and formatting bias (e.g. has a preference for how the judgement is presented to the LLM). Another limitation is the assumption of one true judge that reflects the preferences and values the solution requires. In the case of seeking feedback on answers that appear natural we can imagine many LLMs are capable of this task, however if the task is specialised or domain dependant the question of \textit{which} LLM can act as the judge may be unknown. There exist many models that have been specialised in specific domains such as law \cite{chalkidis2020legal,shaghaghian2020customizing}, finance \cite{wu2023bloomberggpt} and medicine \cite{thirunavukarasu2023large}.

%% file: text/methods.tex
\section{Methods}
\label{sec:meth}
We now propose two methods based on the triplet idea to rank a given set of models when no ground-truth is available for the given dataset. The first one is a more efficient greedy approach, while the latter considers all triplets and performs an evolving weighted majority based ranking.

\subsection{Greedy Triplet Ranking}
Given a set of $n$ models to rank, a dataset of prompts (with no responses) and an evaluation function that compares the model responses with the judge model, the Greedy Triplet Ranking (GTR)
method in Algorithm \ref{algo:gtr} outputs a ranking of these $n$ models by incrementally relegating the worst model in a triplet. In particular, starting at an arbitrary triplet it identifies the worst model, removes it from the triplet, then adds the next model in the set. It does this over all models, where at the end we would have the top two models (i.e. first time the \emph{for} loop is executed). Now with the remaining models it repeats this process again finding the top two in the remaining set. This continues until we have less than three models. One of the top two from the first run is (randomly) selected as the best\footnote{One could also use ranking from the final triplet comparison to determine the best.} and this model is used to resolve the ordering of the pairs of top models that are outputted at the end of the subsequent \emph{for} loop runs.

In this way all the models get ranked at the end, where in each triplet comparison we greedily relegate the worst model to the following run.

\begin{algorithm}[t]
\SetAlgoLined
\textbf{Input:} $M$ set of $(\ge 3)$ models to rank, $f$ evaluation function and dataset $D$ of prompts (viz. questions, instructions, etc.).

\textbf{Initialize:} $R\leftarrow \phi$ \# Current ranked list of models

\While{$|M|\ge 3$}
{
$T = \{M_1,M_2,M_3\}$ \# Triplet to begin

\For{$i=3$ to $|M|$}
{
$T \leftarrow T \cup M_i$

$W \leftarrow$ Using $f(T,D)$ compare pairs of models with the third being the judge and output the model voted as the worst by the other two. In case of tie output $M_i$.

$T \leftarrow T \setminus W$
}
$M\leftarrow M\setminus T$ \# Remove top two models as they will be added to $R$

\eIf{$R==\phi$}
{$R\leftarrow list(T)$ \# Amongst top two models randomly pick one as the best.}
{
$T \leftarrow T \cup R_1$

$W \leftarrow$  Using $f(T,D)$ and $R_1$ as judge output the worse model.

$R \leftarrow $ append $(T\setminus W, W)$ to $R$
}
}
\uIf{$|M|==1$}
{
$R \leftarrow$ append $M$ to $R$
}
\uElseIf{$|M|==2$}
{$T \leftarrow M \cup R_1$

$W \leftarrow$  Using $f(T,D)$ and $R_1$ as judge output the worse model.

$R \leftarrow $ append $(T\setminus W, W)$ to $R$}
\KwOut{$R$}

\caption{Greedy Triplet Ranking (GTR).}
  \label{algo:gtr}
  
\end{algorithm}

\subsection{Full Triplet Ranking}
\label{subsec:full}
For Full Triplet Ranking (FTR) too we assume we are given a set of $n$ models, a dataset of prompts and an evaluation function that compares models based on the third being the judge (see Algorithm \ref{alg:ranking}). The difference from GTR is that we consider all triplet of models here and we compute a \emph{reputation score} for each model that is proportional to the number of (triplet) comparisons that they come on top. This score is then used in the next iteration to determine how much importance to give to a judge in the triplet comparisons, which in turn determines the reputation of other models in that iteration. 

This process continues until the reputations do not change by much. The final reputations are then used to rank the models.

\begin{algorithm}[htbp]
\SetAlgoLined
\textbf{Input:} $M$ set of $(\ge 3)$ models to rank, $g$ triplet evaluation function, dataset $D$ of prompts and $\epsilon$ a small constant.
%and dataset $D$ with $p$ prompts (viz. questions, instructions, etc.).

Evaluate all triplets on $D$ (i.e. $M_i$-vs-$M_j$ using $M_k$) $y_{ijk} = g(\{M_i, M_j, M_k\},D), \forall i,j,k\in M$.

\textbf{Initialize:} Model reputation score $r_k = 1.0 \, \forall k\in M$.

\While{True}
{
    $m_{ij} = \frac{1}{|M|}\sum_k y_{ijk} r_k \,\forall M_i, M_j\in M$   \# Weighted preference matrix.
    
    $z_{ij} = 1$ if $m_{ij} \ge m_{ji} \,\forall M_i, M_j\in M$ and $0$ otherwise \# Majority vote.
    
     $r^\prime_i = \frac{1}{|M| -1} \sum_j z_{ij} \,\forall M_i\in M$ \# New reputation score.
     
    $\delta = \sum_k |r_k - r^\prime_k|$ \# Convergence condition.
    
    \eIf{$\delta\le \epsilon$}
        {
            Break
        }
        {
            $r_k = r^\prime_k$ \# Update reputation
        }
}
Sort $M$ using reputation score $r_k$ for ranking $R$.

\KwOut{$R$}

\caption{Full Triplet Ranking (FTR).}\label{alg:ranking}
\end{algorithm}

\subsection{Analysis}
\label{subsec:analysis}
We now analyze our triplet approach, which also provides intuition of when the approach is likely to work well in practice. We then analyze the time complexity of GTR and FTR.

\subsubsection{Conditions for the triplet approach to succeed}

Given a triplet of models $(M_i,M_j,M_k)$ and their corresponding (fractional) accuracies $(a_i,a_j,a_k)$ on a task, where w.l.o.g. assume $1\ge a_i>a_j>a_k\ge 0$ we want to analyze sufficient conditions under which model $M_k$ will be voted as the worse model by both $M_i$ and $M_j$ as judges. We analyze sufficient conditions since, given simply the ordered accuracies there are no specific necessary conditions on the accuracies for $M_k$ to be voted the worst as it could happen for all possible values.

First we analyze the case where none of the models will agree on incorrect responses to input prompts. In other words, no two models will have the same incorrect response on an input. If they are incorrect they will be incorrect in different ways.

\begin{lemma}
\label{lem1}
    Given a triplet of models $(M_i,M_j,M_k)$, where their accuracies $(a_i,a_j,a_k)$ satisfy $1\ge a_i>a_j>a_k\ge 0$ with no two models agreeing upon incorrect responses, then $a_k < a_i+a_j-1$ will result in $M_k$ being (correctly) voted as the worse model by both $M_i$ and $M_j$ as judges.
\end{lemma}
\begin{proof}
    We want to get the least overlap between correct responses of $M_i$ and $M_j$. This will happen when $M_j$ is correct on all the responses that $M_i$ is incorrect on, which is $1-a_i$ fraction. Note that $M_j$ can also be correct on all the responses that $M_i$ is correct on. The overlap then between $M_i$ and $M_j$ of correct responses is $a_j-(1-a_i)=a_j+a_i-1$. Hence, if $a_k < a_j+a_i-1$ then both $M_i$ and $M_j$ will rate $M_k$ to be worse than the other even if $M_k$ overlaps on all correct responses with $M_i$ and/or $M_j$.
\end{proof}

Note that the above is a sufficient condition so in practice our approach can work even if the above condition is not met as is seen in the experiments. This is especially true when $a_i+a_j \le 1$, where in practice we might still be able to identify $M_k$ as the worst, although the sufficient condition would require $a_k<0$.

If we allow two models to agree on up to $m$ fraction of the incorrect responses then the above result can be extended as follows:
\begin{theorem}
\label{thm1}
    Given the setup in Lemma \ref{lem1}, but where two models may agree on up to $m$ fraction of the incorrect responses, then $a_k < a_i+a_j-1-m$ will result in $M_k$ being (correctly) voted as the worse model by both $M_i$ and $M_j$ as judges.
\end{theorem}
\begin{proof}
    From Lemma \ref{lem1} we know the minimum overlap between $M_i$ and $M_j$ can be $a_i+a_j-1$. However, $M_k$ could agree on $m$ fraction of incorrect responses with either $M_i$ or $M_j$ making its maximum possible agreement with $M_i$ or $M_j$ to be $a_k+m$. Hence, if $M_k$ is to be voted as the worst $a_k+m < a_i+a_j-1$.
\end{proof}

\emph{As such, we would expect $m$ to be small for LLMs, as they are typically used for generative tasks where the generations/responses can be multiple sentences or paragraphs (viz. summarization, writing articles, etc.) leading to a diverse set of possible responses and hence many different ways in which they may be incorrect.} In other words, there can be a plethora of ways in which they may respond incorrectly when they make mistakes and hence, overlap between incorrect responses for different LLMs is likely to be small, especially when they are from different model families, for many applications.

Although the sufficient conditions above might be a bit strict they provide the intuition, also seen in practice, that there needs to be some gap in performance between the best and the worst models for our schemes to work. This seems to be reasonable to have in practice, since if all models are of similar accuracy, then all rankings can be considered as equally good/bad.

We now prove that if the sufficient conditions in Theorem \ref{thm1} are met, then our GTR algorithm will output the correct ranking where the ranking of the top two models may be flipped.

\begin{proposition}
Given a set $M$ of models ($|M|\ge 3$) with a strict (accuracy) ordering on some dataset with the condition on accuracies in Theorem \ref{thm1} holding for all triplet of models in $M$, then the GTR algorithm will output the correct ranking, where the order of the top two models may be flipped with $\frac{1}{2}$ probability.
\end{proposition}
\begin{proof}
In the (inner) \emph{for} loop of the GTR algorithm we will always correctly identify the top two models amongst those considered present in $M$ at that iteration of looping. This is because all triplet comparisons will correctly identify the worst model given Theorem \ref{thm1} as we iterate through the for loop leaving just the two best models of that iteration at the end. Ordering these top two models when $R\neq \phi$ can be done by $R_1$ correctly again because of Theorem \ref{thm1}. However, when $R=\phi$ we cannot resolve the ordering of the (overall) top two models and so we randomly pick one with probability $\frac{1}{2}$. Moreover, in the first looping of the \emph{for} loop all models undergo a triplet comparison at least once. Outside the \emph{while} loop when $M=1$ we have the worst model as it has lost in all rounds of triplet comparisons. When $M=2$ we can again pick $R_1$ as the judge and because of Theorem \ref{thm1} we should be able to correctly order these last models as well. 
\end{proof}

The FTR algorithm is harder to analyze because it weights the models in each iteration based on the extent to which they were up voted and is a more sophisticated heuristic 
%that we wanted to propose, but 
whose analysis we leave for future work.

\subsubsection{Time Complexity}
The \emph{for} loop in the GTR algorithm identifies the top two models every time it is invoked and removes them from $M$. Thus, if there were $n$ total models in $M$
initially we do $O(n)$ comparisons in the \emph{for} loop. Since the size of $M$ reduces by a constant ($2$) every time we run the \emph{for} loop the total time complexity of GTR is $O(n^2)$.

FTR on the other hand performs all triplet comparisons and hence is $O(n^3)$.

%% file: text/experiments.tex
\section{Experiments}
\label{sec:expts}

% Baseline
We test our methods on three cases: summarization task, multiple-choice, and dialog. As a basis of comparison we devise a new method based on the `most common answer' (MCA). The MCA ranker parses outputs from all LLMs to determine the most common output which is then treated as the reference to rank the models. For contexts such as (single token) multiple choice Q\&A, this is possibly the preferred method as it is easy to aggregate responses and ensembling in such contexts can have state-of-the-art performance \cite{ensem}.

% Measure of effectiveness
To measure effectiveness of the methods, we compare the estimated rankings with `true' rankings derived using reference data. We use two metrics: Rank-biased Overlap (RBO) \cite{webber2010similarity} and Mean Average Precision@k (MAP-$k$) to quantify how well the estimated rankings reflect the `true' rankings. Both measures span the unit interval, with $1.0$ denoting perfect agreement. 

\paragraph{Summarization:} We test on two datasets: the CNN-Daily Mail (CNN/DM) summarization dataset \cite{hermann2015teaching, nallapati-etal-2016-abstractive}  and XSUM \cite{narayan2018don}. In both cases we gather responses from $40$ LLMs for $3000$ instances from HELM \cite{liang2023holistic}. Our evaluation functions, i.e. $f$ in GTR and $g$ in FTR, employ ROUGE scores which are widely used for summarization tasks and in HELM. In particular, we consider the f-measure of bigrams. For the `most-common answer', we use the top-$k$ bigrams from all model responses and consider this to be the reference answer ($k=256$ in our experiments). For this case, we also compare with LLM-as-a-judge approach using Prometheus-2 \cite{kim2024prometheus}, which is a state-of-the-art LLM evaluator or judge model. True ranking, against which our estimated ranking is compared, is inferred by comparing model outputs to the reference response in the dataset.
%we experimented with entailment probabilities using DeBerta \cite{he2020deberta} and ROUGE-2 scores. 

\paragraph{Multiple-choice:} We simulate responses from models with known accuracy in a multiple choice setting. Datasets are constructed for several cases by varying the number of models, model accuracy, number of prompts and number of possible answers. Table \ref{tab:app:synth} shows a sample. Since answers are discrete, the evaluation functions (i.e. $f$ in GTR and $g$ in FTR) use the equality operator to determine if a response is correct. We also experimented with a noisy equality operator, where the outcome of the evaluation is flipped with a known probability. This simulates a realistic case where comparisons are imperfect.

\paragraph{Dialog: } In our third case, we tackle a practical need of an internal team comparing several fine-tuned variants that optimize for multiple objectives. For this task, they gathered responses from four model variants for $100$ prompts from the Moral Integrity Corpus \cite{ziems2022moral}. The corpus contains annotations of chatbot logs that describe their latent values and moral statements. These responses were passed on to three human annotators who were tasked with determining the `best' response from the four models. Annotators provided a single preference per response.

We summarize the main empirical findings next with Appendix \ref{sec:app:details} providing more details on experimental setup and results.

\subsection{Results}
\label{subsec:results}

\paragraph{Summarization:} Overall, both GTR and FTR outperform the MCA as shown in Table \ref{tab:exp:summary}. The methods improve in recovering ranking as the size of the prompt dataset increases. In other words, both methods benefit from a larger number of prompts to make accurate evaluations. 

We experimented with the number of models to rank by sampling from the catalogue of $40$ models. Additional sampling experiments, where we considered mode performance, are reported in the appendix. Figure \ref{fig:exp:summarization:metrics} show rankings to be reliably recovered for most cases (recall an RBO value of 1.0 implies perfect agreement with true rankings). Quality does degrade as the number of models to be ranked increases. 

The number of triplet evaluations are shown in Figure \ref{fig:exp:evaluations}. The greedy approach requires far fewer evaluations to estimate rankings as compared to the full method. This is not a concern when evaluations are cheap (e.g. ROUGE scores), but can be prohibitive when they are not. For example, our experiments leveraging entailment probabilities using DeBerta \cite{he2020deberta} were scuttled due to expensive evaluations. In such cases, the greedy method is preferred.

Additional details on experiments are shown in Appendix \ref{sec:app:details} where we show the impact of prompt dataset size (larger datasets aid evaluations particularly for FTR). Additionally, we experiment with various sampling strategies to study the effect of model quality differences in ranking. We show empirically that models with discernible quality differences are easier to rank. We also note that the MCA is not the most natural approach in such contexts where the responses are free form, although based on our design of using bi-grams it was easily compatible with existing metrics like ROUGE.

% Things to say:
% \begin{itemize}
%     \item Both methods outperform baseline.Table \ref{tab:exp:summary} 
%     \item Table \ref{tab:exp:summary} also shows aggregate measures of ranking performance as a function of dataset size. Both methods benefit from larger number of prompts to make triplet evaluations
%     \item Figure \ref{fig:exp:summarization:metrics}  show rankings to be reliably recovered
%     \item Efficient evaluations for greedy \ref{fig:exp:evaluations}
%     \item Summarise findings from appendix: results as a function of model quality, 
% \end{itemize}

\begin{figure}[ht]
     \centering
     \begin{subfigure}[b]{0.22\textwidth}
         \centering
         \includegraphics[width=\textwidth]{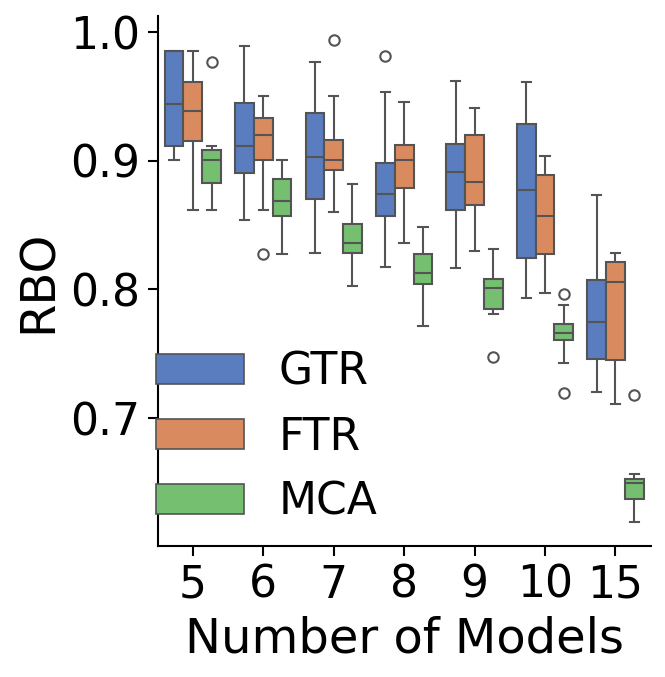}
     \end{subfigure}
     \begin{subfigure}[b]{0.22\textwidth}
         \centering
         \includegraphics[width=\textwidth]{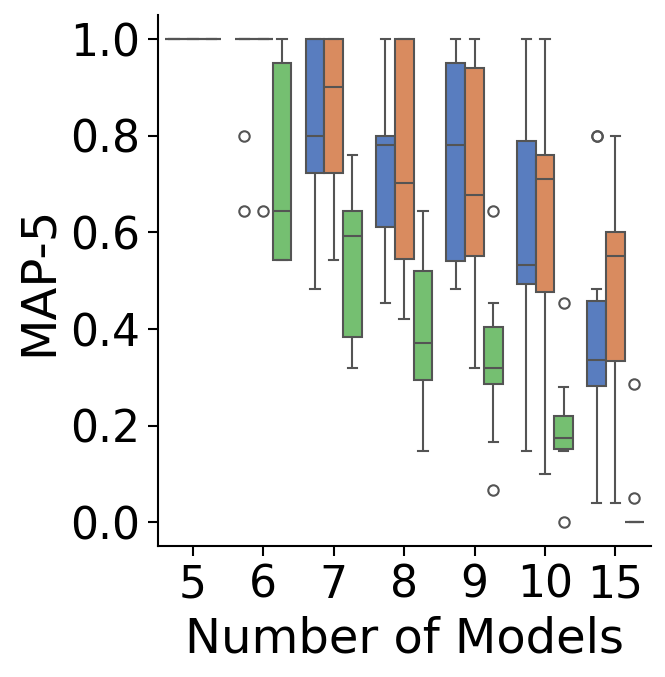}
     \end{subfigure}
    \begin{subfigure}[b]{0.22\textwidth}
         \centering
         \includegraphics[width=\textwidth]{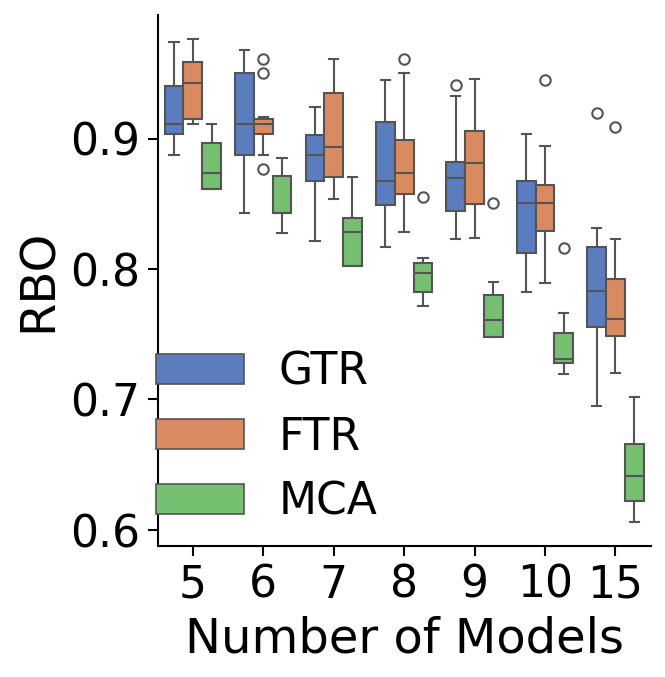}
     \end{subfigure}
     \begin{subfigure}[b]{0.22\textwidth}
         \centering
         \includegraphics[width=\textwidth]{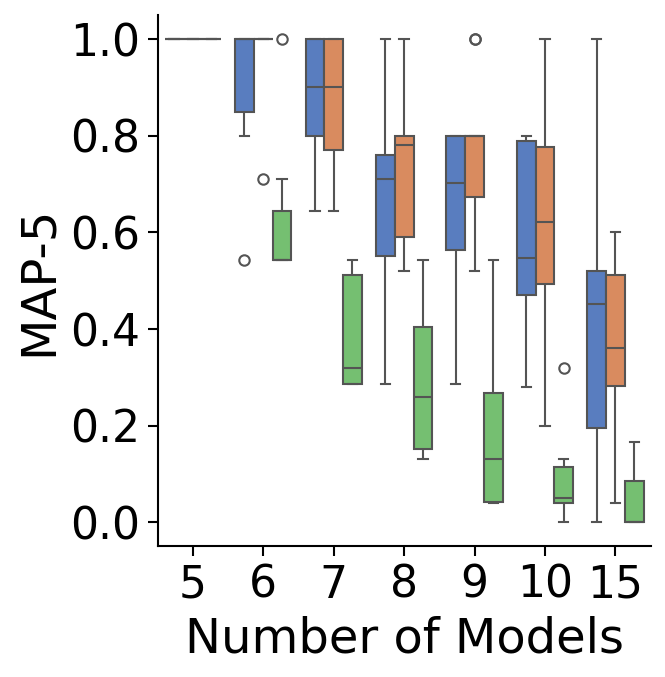}
     \end{subfigure}
        \caption{Evaluation metrics on summarization for two datasets: CNN/DM (top) and XSUM (bottom), RBO (left) and MAP-5 (right), as a function of number of models being ranked (note x-axis is \emph{not} linear).}
        \label{fig:exp:summarization:metrics}
        %\vspace{-.25cm}
\end{figure}

% \begin{figure}[ht]
%     \centering
%     \includegraphics[width=0.8\columnwidth]{figures/su-RBO-main.png}
%     \caption{RBO}
%     \label{fig:exp:summarization:rbo}
% \end{figure}

% \begin{figure}[ht]
%     \centering
%     \includegraphics[width=0.8\columnwidth]{figures/summ-MAP-5-main.png}
%     \caption{MAP5}
%     \label{fig:exp:summarization:map5}
% \end{figure}

\begin{figure}[ht]
    \centering
    \includegraphics[width=.8\columnwidth]{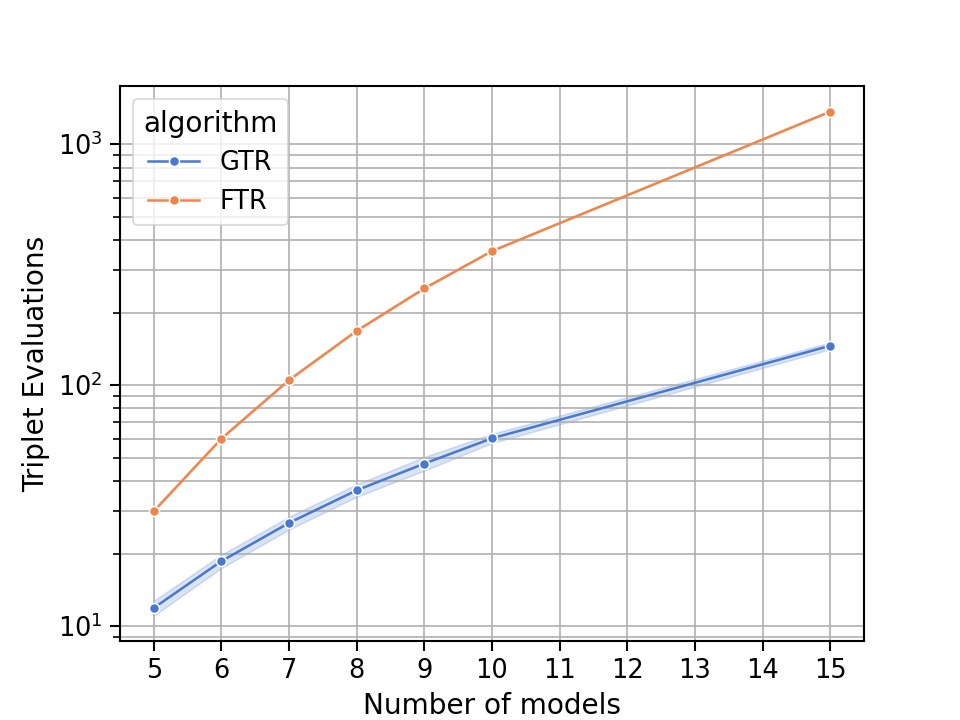}
    \caption{Number of triplet evaluations for CNN/DM dataset. (log y-scale).}
    \label{fig:exp:evaluations}
    %\vspace{-.25cm}
\end{figure}

\begin{table*}[ht]
    \centering
    \caption{Mean and standard deviation of Rank-biased overlap (RBO) and Mean Average Precision at $5$ (MAP-5) for varying sizes of prompts used for the summarization task. Computed over $10$ trials. Best (mean) result in bold.}
    \small
\begin{tabular}{l|l|lll|lll}
\toprule
 & & \multicolumn{3}{c}{RBO} & \multicolumn{3}{|c}{MAP-5} \\
& Size & MCA & FTR & GTR & MCA & FTR & GTR \\ \midrule
CNN/& 100 & 0.808$\pm$0.07 & 0.866$\pm$0.07 & \bf{0.867$\pm$0.07} & 0.445$\pm$0.33 & \bf{0.677$\pm$0.29} & 0.664$\pm$0.28 \\
DM&500 & 0.802$\pm$0.07 & 0.881$\pm$0.05 & \bf{0.882$\pm$0.05} & 0.430$\pm$0.33 & \bf{0.804$\pm$0.22} & 0.750$\pm$0.22 \\
&1000 & 0.804$\pm$0.07 & \bf{0.893$\pm$0.06} & 0.884$\pm$0.06 & 0.425$\pm$0.33 & \bf{0.815$\pm$0.23} & 0.763$\pm$0.23 \\
&3000 & 0.805$\pm$0.08 & \bf{0.889$\pm$0.05} & 0.885$\pm$0.07 & 0.460$\pm$0.33 & \bf{0.782$\pm$0.26} & 0.746$\pm$0.26 \\ \midrule
XSUM&100 & 0.803$\pm$0.07 & \bf{0.878$\pm$0.06} & 0.860$\pm$0.06 & 0.433$\pm$0.32 & \bf{0.744$\pm$0.24} & 0.723$\pm$0.26 \\
&500 & 0.799$\pm$0.07 & \bf{0.884$\pm$0.06} & 0.871$\pm$0.07 & 0.420$\pm$0.32 & \bf{0.758$\pm$0.25} & 0.716$\pm$0.26 \\
&1000 & 0.790$\pm$0.07 & \bf{0.888$\pm$0.06} & 0.877$\pm$0.06 & 0.374$\pm$0.33 & \bf{0.785$\pm$0.27} & 0.739$\pm$0.25 \\
&3000 & 0.788$\pm$0.08 & \bf{0.879$\pm$0.06} & 0.871$\pm$0.06 & 0.371$\pm$0.34 & \bf{0.761$\pm$0.26} & 0.731$\pm$0.26 \\
\bottomrule
\end{tabular}
    \label{tab:exp:summary}
\end{table*}

\paragraph{Multiple-choice: } The efficacy of our methods for multiple choice responses is summarized in Figure \ref{fig:exp:synth:metrics}. Interestingly, all three methods perform very poorly when the number of possible responses are low. To see why this is the case, consider a triplet evaluation with yes/no responses. A weak model that erroneously outputs `no' will judge another model that responds with `no' as better. The judge's response is highly correlated with weak models' outputs. In other words, as there are few failure modes and low variance in wrong answers, weak models tend to be promoted. Therefore, triplet evaluations are susceptible to cases where there are few possible outcomes. This validates the theory in Section \ref{subsec:analysis}, where we mentioned that we need $m$ (overlap of incorrect responses) to be small for our idea to be effective.

With increase in the size of the response set MCA, as expected, is highly performant as it ensembles responses from all the evaluated models to come up with an answer. For discrete responses that are easy to consolidate, such as single token responses in multiple choice, the most common answer is also likely to be the right answer. So this serves as a good proxy even when reference data is not available. Nonetheless, our triplet methods, especially FTR, still are quite competitive even in this case showcasing the generality of contexts in which they can be used.

Since the data in this case is simulated, we tested the impact of model performance on ranking recovery. Table \ref{tab:exp:synth:summary} shows that even when the correctness of outcomes are pure chance (i.e. 50\%) for the best of the models being ranked, the RBO metric is at $0.832$ for FTR and $0.723$ for GTR which indicate good agreement with true rankings.

% Things to say:
% \begin{itemize}
%     \item Need variance in wrong answers for this to work.
%     \item Baseline based on most-common answer is very strong.
% \end{itemize}

\begin{table*}[ht]
    \centering
 \caption{Mean and standard deviation of Rank-biased overlap (RBO) and Mean Average Precision at $5$ (MAP-5) for varying accuracy of the best performing model for multiple-choice. Computed over $5$ trials. Best (mean) result in bold.}
 \small
\begin{tabular}{l|lll|lll}
\toprule
 Acc. best  & \multicolumn{3}{c}{RBO} & \multicolumn{3}{|c}{MAP-5} \\
model & MCA & FTR & GTR & MCA & FTR & GTR \\ \midrule
30 & 0.668$\pm$0.26 & \bf{0.694$\pm$0.26} & 0.622$\pm$0.21 & 0.378$\pm$0.39 & \bf{0.433$\pm$0.41} & 0.262$\pm$0.28 \\
50 & 0.818$\pm$0.23 & \bf{0.832$\pm$0.22} & 0.723$\pm$0.19 & 0.663$\pm$0.39 & \bf{0.698$\pm$0.38} & 0.452$\pm$0.33 \\
70 & \bf{0.927$\pm$0.16} & \bf{0.927$\pm$0.15} & 0.833$\pm$0.14 & \bf{0.866$\pm$0.29} & \bf{0.866$\pm$0.27} & 0.671$\pm$0.30 \\
90 & 0.980$\pm$0.09 & \bf{0.981$\pm$0.08} & 0.919$\pm$0.09 & 0.965$\pm$0.17 & \bf{0.971$\pm$0.15} & 0.855$\pm$0.20 \\
\bottomrule
\end{tabular}
   \label{tab:exp:synth:summary}
\end{table*}

We investigate the role of noise in the evaluation function. In multiple choice, this comparison is exact as it uses an equality operator. This is substituted with a \emph{noisy} equality operator, which randomly flips the outcome with a noise probability $p$. This allows us to evaluate the robustness of our methods in the face of noisy evaluations - which is a more realistic setting in generative cases. Figure \ref{figure:main:noise} shows the RBO metric as a function of noise for various methods. Both FTR and GTR are much more robust to noise at low and medium levels as compared to MCA. FTR does best here with its dynamic weighted voting making it less susceptible to noise. MCA on the other hand starts to degrade rapidly as noise increases. Additional plots are presented in Figure \ref{fig:app:synth:noise} in the appendix. This suggests our algorithms have broader applicability, including for single token classification type tasks.

%Moreover, as seen in Figure \ref{fig:app:synth:noise} in the appendix, FTR and even GTR are more robust than MCA at low-medium levels of noise. This suggests that our algorithms could still be used for non-generative or single token multiple choice kind of tasks.

\begin{figure*}[ht]
    \centering
    \includegraphics[width=0.85\textwidth]{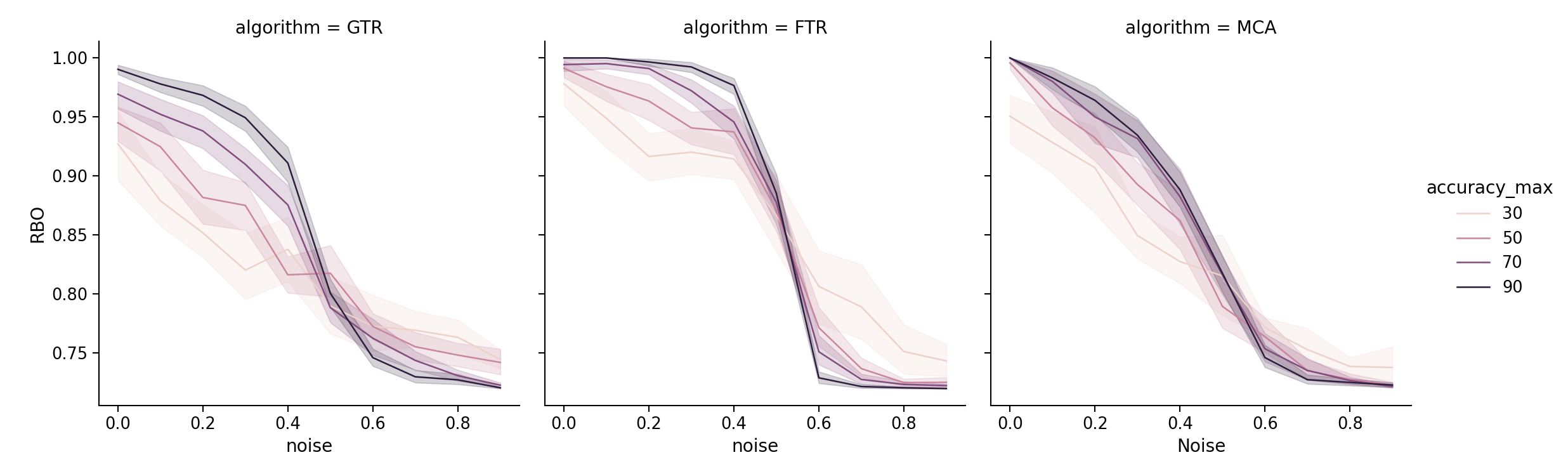}
    \caption{Quality of rankings recovered as a function of noise in the evaluation function for different methods. Four sets of $10$ LLMs with each set having a maximum LLM accuracy of $30\%$, $50\%$, $70\%$ or $90\%$ are considered, where the number of questions is $50$. We see the robustness of the proposed methods (GTR and FTR) at low to medium levels of noise in such a setup.}
    \label{figure:main:noise}
\end{figure*}

\begin{figure}[htb]
     \centering
     \begin{subfigure}[b]{0.22\textwidth}
         \centering
         \includegraphics[width=\textwidth]{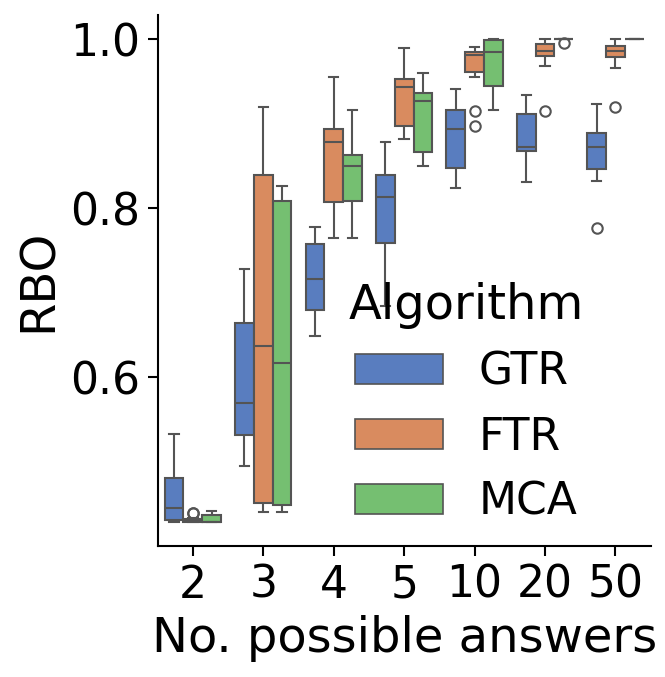}
     \end{subfigure}
     \begin{subfigure}[b]{0.22\textwidth}
         \centering
         \includegraphics[width=\textwidth]{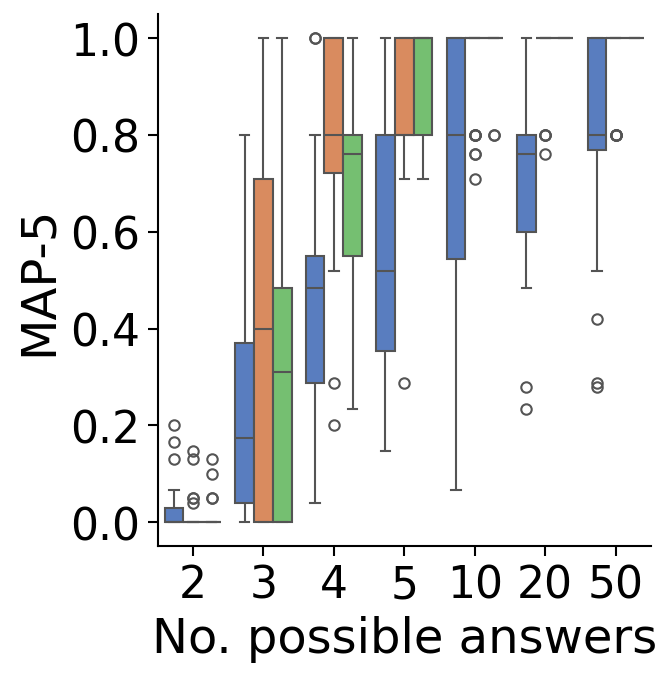}
     \end{subfigure}
        \caption{Evaluation metrics on multiple-choice, RBO (left) and MAP-5 (right), when ranking $25$ models where the accuracy of the best performing model is 50\%.}
        \label{fig:exp:synth:metrics}
\end{figure}

% \begin{figure}[ht]
%     \centering
%     \includegraphics[width=0.8\columnwidth]{figures/synth-RBO-main.png}
%     \caption{RBO}
%     \label{fig:exp:synth:rbo}
% \end{figure}

% \begin{figure}[ht]
%     \centering
%     \includegraphics[width=0.8\columnwidth]{figures/synth-MAP-5-main.png}
%     \caption{MAP5}
%     \label{fig:exp:synth:map5}
% \end{figure}

\paragraph{Dialog:} Our last use case is qualitative and ground in a practical application of assessing amongst several fine-tuned variants. Specifically, from a pre-trained model ($PT$), several reward functions were specified that aimed to fine tune for multiple moral objectives. Three variants ($M_1$, $M_2$, and $M_3$) were LLMs created by fine-tuning for each reward function using RLHF. Outputs of each model for $100$ prompts were recorded and scored by human annotators. The annotator task was simply to choose the `best' response, without qualifying any specific dimension for evaluation. Three annotators unrelated to this paper specified their preferences. We evaluated the FTR method using entailment probability and ROUGE scores as the basis for the evaluation function $g$. We ran only FTR here, since there were only four models and so efficiency wasn't a factor. Our method ranked the models $[M_1, M_3, M_2, PT]$ as compared to the human annotators $[M_1, M_2, M_3, PT]$. Thus, we were able to rank the best and the worst models correctly. The middle two were quite close in performance from what we heard so flipping them was probably not that damaging.

\paragraph{Comparison with LLM-as-a-judge} Next, we compare our approach to the growing LLM-as-a-judge literature \cite{zheng2024judging}. In this setting, language models themselves are used as evaluators. LLM judges used in this was offer lower cost of annotations/labels compared to human annotations and can make metric-free evaluations, for instance evaluating responses that meet criteria such as `child-friendliness' or `professional tone'. Prometheus \cite{kim2023prometheus} and its subsequent version \cite{kim2024prometheus} are two such specialized LLM evaluators. 

We leverage Prometheus-2 to compare LLM outputs in a pairwise fashion.
As such, we prompt it to compare the results of two summaries from two LLMs for the CNN/DM summarization dataset. This is done for all pairs of LLMs and for a sample of $50$ questions in the benchmark. For each evaluation, the LLM judge (i.e. Prometheus-2) declares a winning LLM. We aggregate win rates across the benchmark and 
derive a LLM ranking based on win rates. We consider a (random) sample of $10$ LLMs from the benchmark. Our task criteria is simply `Which is the better response?'.

Table \ref{tab:exp:llmj} shows the mean and standard deviation of ranking metrics compared to this approach. Our methods out-perform the LLM judge across all metrics. Additional details are included in the Appendix \ref{subsec:app:prometheus}. 

\begin{table}[]
    \centering
    \small
     \caption{Comparison with state-of-the-art Prometheus-2 LLM-judge using pairwise evaluation on CNN/DM summarization task with $10$ models. Best (mean) result in bold.}
     \begin{tabular}{l|lll}
\toprule
 Algorithm & RBO & MAP-3 & MAP-5 \\
\midrule
 FTR & 0.835$\pm$0.04 & 0.200$\pm$0.14 & \bf{0.523$\pm$0.16} \\
GTR & \bf{0.841$\pm$0.06} & \bf{0.300$\pm$0.26} & 0.448$\pm$0.23 \\
Prometheus-2 & 0.806$\pm$0.07 & 0.200$\pm$0.18 & 0.354$\pm$0.28 \\
MCA & 0.792$\pm$0.05 & 0.144$\pm$0.16 & 0.296$\pm$0.26 \\
\bottomrule
\end{tabular}
    \label{tab:exp:llmj}
\end{table}

%% file: text/concl.tex
\section{Discussion}
\label{sec:discuss}
Given the proliferation of LLMs, and claims of superiority made by the developers, there needs to be a trustworthy mechanism to evaluate them. LLM leaderboards are extremely beneficial for this purpose. However, it takes a lot of effort to collect reference data or judgements or preferences about the model responses, which only multiplies when we consider various domains and tasks which can also lead to benchmarks becoming obsolete. Our proposed approaches can be seen as a first pass to substantially reduce this effort without imposing the need to have a trustworthy model for evaluation as demanded by LLM-as-a-judge approaches. Moreover, in addition to simply obtaining a ranking, the relative values of the reputation score for different LLMs in FTR could potentially be used to also gauge the performance gap between them.
%\begin{itemize}
%\item sample complexity - how many examples do we need to get a good ranking - may be future work - can use small set of examples from existing data and see how good the rankings are
%\item partial ranking of models - future work
%\item partial access to labels - future work
%\item: experiment: When we present the summarization results, can we say that the baseline of picking the most common bigrams is NOT the most natural ways to compute the "best summary".
%\end{itemize}

There are many interesting avenues to explore in the future. First, how we can incorporate and benefit from additional information such as a partial ordering between LLMs, which is reasonable to assume when some LLMs belong to a model family. The additional information could also be in the form of a few ground truth labels for the task at hand. In this case, there may be a wider set of possible methods such as those based on uncertainty or possibly even Bayesian approaches. Second, rather than triplets what if we consider larger sets to compare in each round. Of course, for FTR the number of comparisons might scale exponentially with the set size so more efficient variants might have to be thought of. Here it is possible that the variance in ranking the worse models might increase as a larger set of judges will produce  their own ranking. Moreover, asking each judge to rank a larger set of models might be more error prone. Nonetheless, the bias-variance trade-off between different set sizes might be interesting to explore. Third, more efficient variants than GTR may be worth exploring. For instance, in GTR the top two models at the end of the first round may be used to rank the rest resulting in a linear time variant. Also methods from matrix completion could be explored where we obtain a few partial orderings running GTR or FTR for a few iterations. Fourth, using a threshold in the triplet comparisons to determine a winner may reduce the possibility of selecting a winner based on random chance and further improve the ranking. Finally, our approach might find use in other contexts where, consolidation of responses into a single response is not easy such as: i) when evaluating non-expert human annotators on descriptive tasks. ii) In accelerated drug discovery type of applications where different algorithms might recommend different molecular structures that may be hard to consolidate or iii) even in multi-agent systems where a protocol is needed to decide which agent to follow in a specific context.

%% file: text/appendix.tex
\appendix

\section{Experiment Details}
\label{sec:app:details}

\subsection{Summarization}
\label{subsec:app:summ}

% Dataset
We use summarization scenario with CNN/Daily Mail and XSUM and collect instance runs from HELM v0.2.3. Specifically, we use HELM APIs to gather the per instance responses of all benchmarked models ($40$ in all). For each summarization prompt, HELM runs three trials using various perturbations on the prompt. The resulting dataset consists of $1000$ summarization queries, each run with $3$ perturbations, resulting in $3000$ prompt-response pairs for each of the $40$ models. Each row in this dataset also contains the reference output used to compute the true rankings.

% Experiment design
For our experiments, we take $n$ samples from the set of models, $n=[5, 6, 7, 8, 9, 10, 15]$ in our experiments. For each $n$ we conduct $10$ trials. Three sampling strategies are used. 
\begin{itemize}
    \item Random: $n$ models are sampled at random.
    \item Spread: We first compute true rankings of all $40$ models using same evaluation function $g$ and the reference output. We then use systematic sampling to selected $n$ models such that model performance is `spread' across the list. Specifically, we first generate from  $U[0, 40/n]$ and then select every $40/n$-th model from the ordered set of models. This provides a sample of models where performance is varied.
    \item Close: We order all models based on true ranking (similar to the case above). We define a window size $w$ and draw a sample $s$ from $U[w, 40 - w]$. We then sample $n$ samples from a subset in the window $U[s-w, s+w]$. This results in a sample of models with similar performance.
\end{itemize}

All results in the main paper are reported based on the random sampling strategy. Figures \ref{fig:app:summ:metrics-vs-sampling} and \ref{fig:app:summ:metrics-vs-sampling:xsum} show the assessment when the sampling strategy changes.

For each case, we vary the number of summarization prompts to be $[100, 500, 1000, 3000]$. The three methods FTR, GTR and the baseline are run on each combination of run configurations for 10 trials each. As evaluation metrics, we compute Rank Biased Overlap (RBO) with $p=0.95$ and Mean Average Precision @ k, with $k = 3, 5, 10$. As the evaluation function, $f$ in GTR and $g$ in FTR, we use the ROUGE-2 fmeasure with is a common summarization metric and also used in HELM.

To complete the number of evaluations for CNN/DM, Figure \ref{fig:app:xsum:eval} shows the number of evaluations for the XSUM summarization task, showing similar scaling. 

\begin{figure}[ht]
    \centering
    \includegraphics[width=\columnwidth]{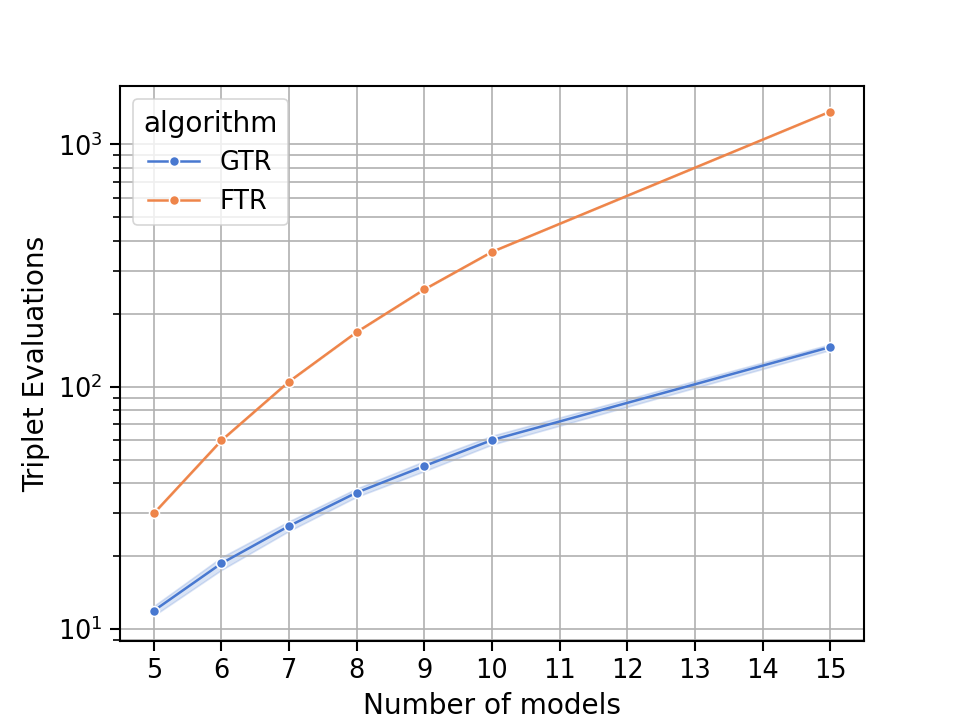}
    \caption{Triplet evaluations for XSUM summarization task.}
    \label{fig:app:xsum:eval}
\end{figure}

Figure \ref{fig:app:summ:metrics-vs-size} shows the impact of prompt size with both FTR and GTR benefitting from a larger corpus of prompts to evaluate over. 

\begin{figure*}[ht]
    \centering
    \begin{subfigure}[b]{\textwidth}
        \centering
        \includegraphics[width=\textwidth]{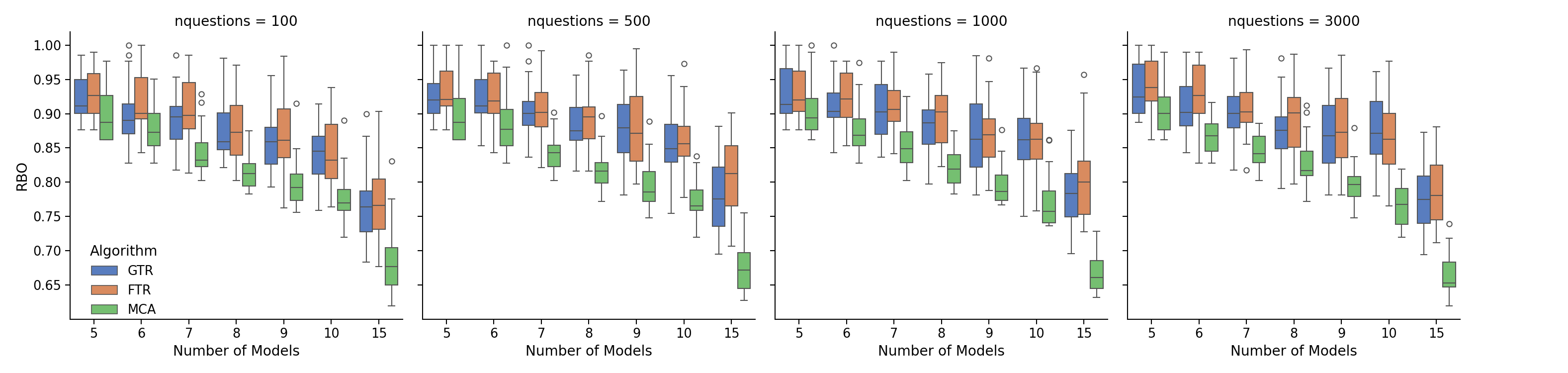}
    \end{subfigure}
    \begin{subfigure}[b]{\textwidth}
        \centering
        \includegraphics[width=\textwidth]{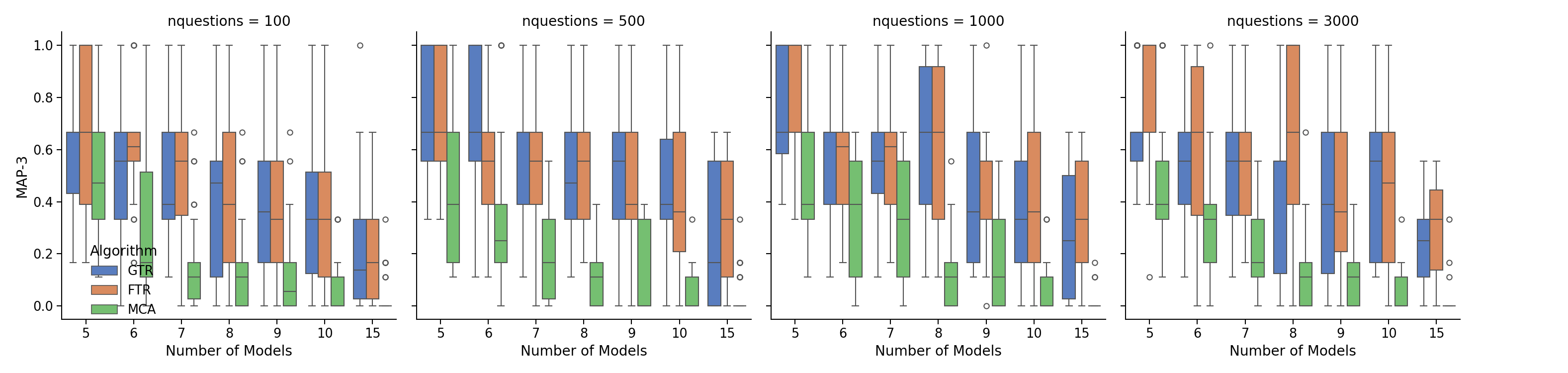}
    \end{subfigure}
    \begin{subfigure}[b]{\textwidth}
            \centering
            \includegraphics[width=\textwidth]{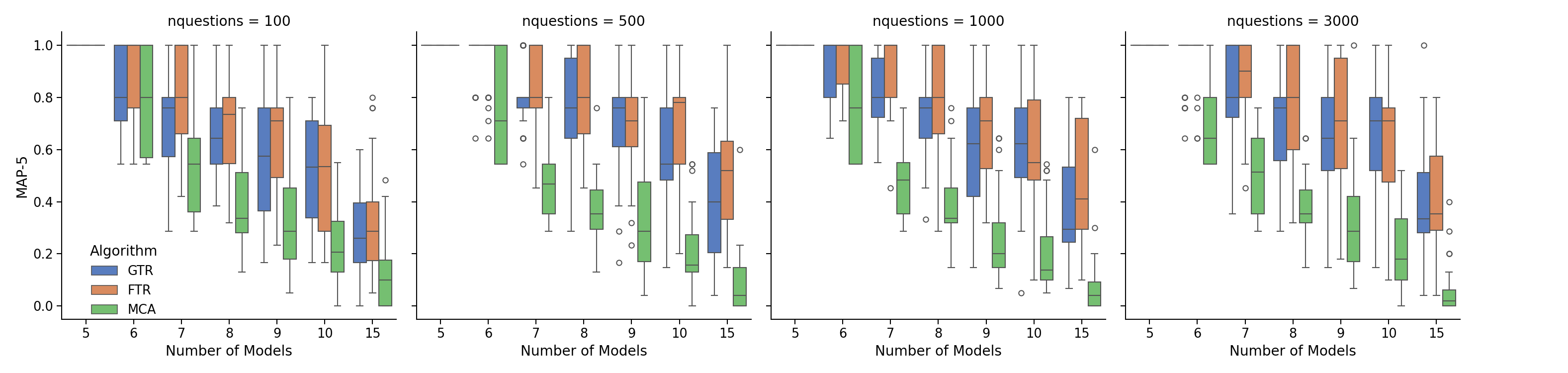}
    \end{subfigure}
       \caption{Summarization - CNN/DM dataset: Evaluation measures as a function of prompt dataset size. RBO (top), MAP-3 (middle), MAP-5 (bottom).}
       \label{fig:app:summ:metrics-vs-size}
\end{figure*}

\begin{figure*}[ht]
    \centering
    \begin{subfigure}[b]{\textwidth}
        \centering
        \includegraphics[width=\textwidth]{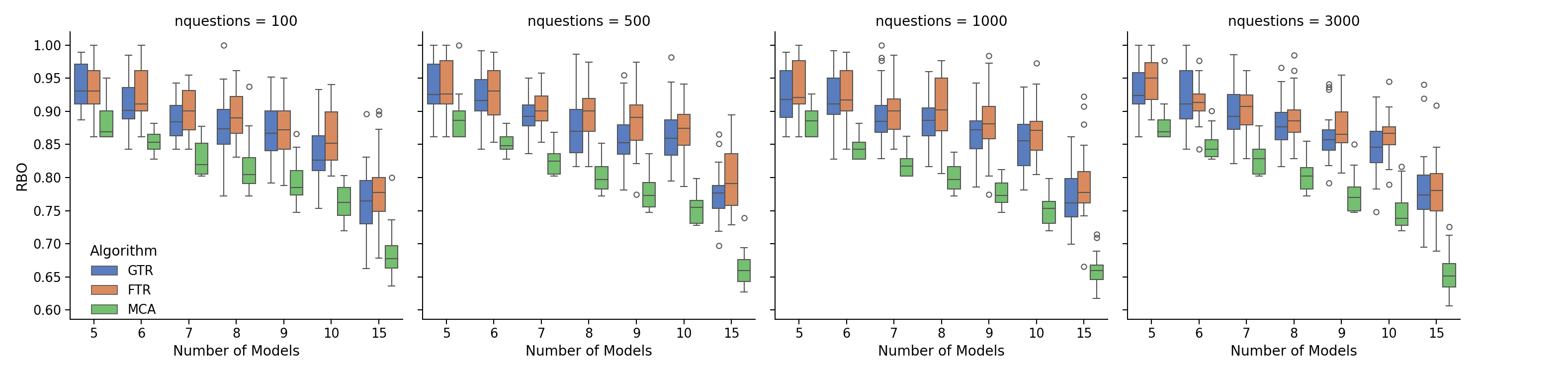}
    \end{subfigure}
    \begin{subfigure}[b]{\textwidth}
        \centering
        \includegraphics[width=\textwidth]{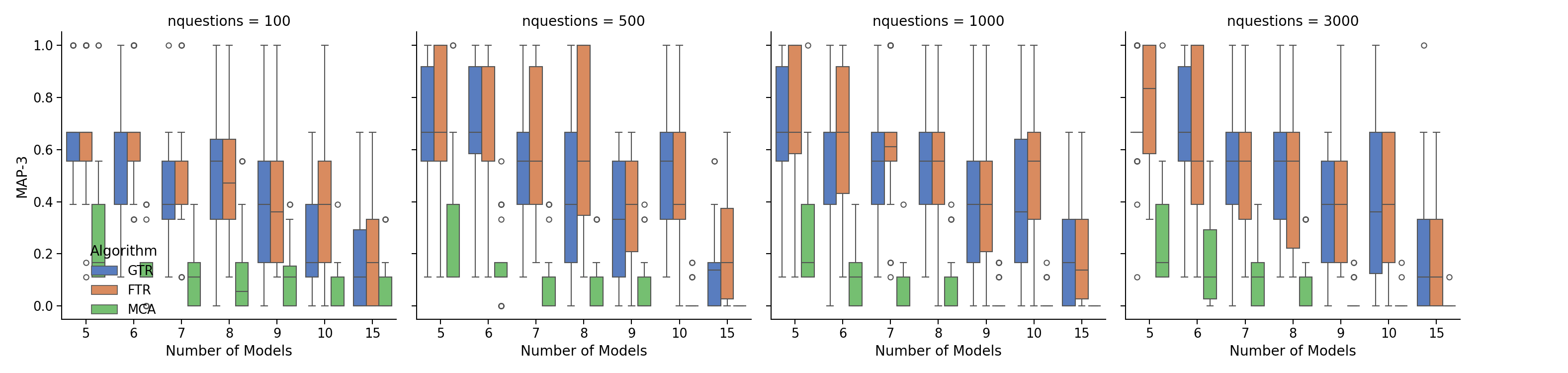}
    \end{subfigure}
    \begin{subfigure}[b]{\textwidth}
            \centering
            \includegraphics[width=\textwidth]{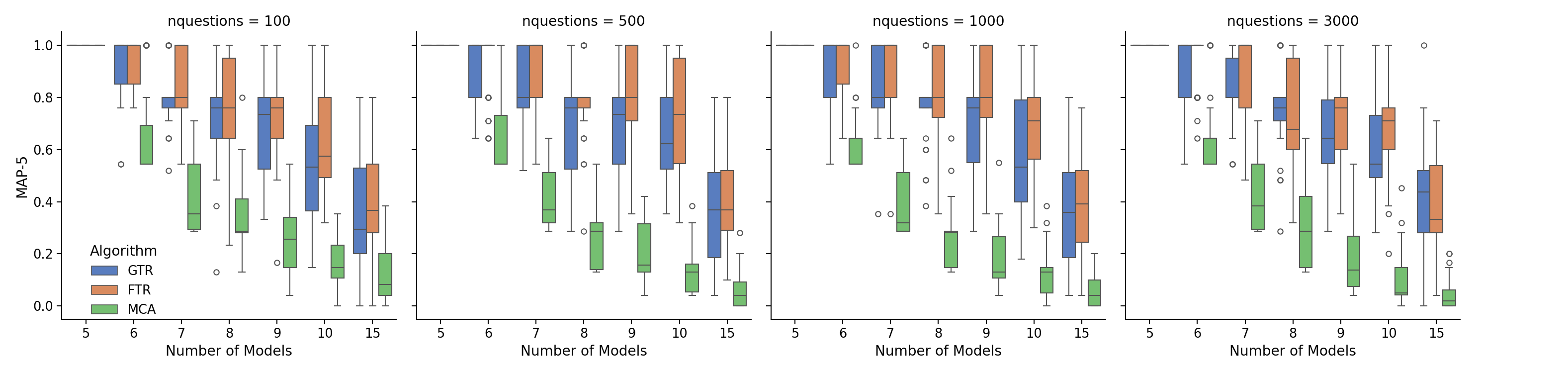}
    \end{subfigure}
       \caption{Summarization - XSUM dataset: Evaluation measures as a function of prompt dataset size. RBO (top), MAP-3 (middle), MAP-5 (bottom).}
       \label{fig:app:summ:metrics-vs-size:xsum}
\end{figure*}

The different sampling strategies yielded different performance characteristics as shown in Figure \ref{fig:app:summ:metrics-vs-sampling}. When model performance is similar (sampling = close), estimating ranking is more challenging as it is harder to distinguish between model responses. All test methods appear to have similar performance in this case. On the other hand, when model perfomances are dissimilar (sampling = spread), rankings are recovered with more efficacy. The FTR model here performs very well. In the random case, the GTR and FTR methods out-perform the most-common-answer method considerably.
\begin{figure*}[ht]
    \centering
    \begin{subfigure}[b]{\textwidth}
        \centering
        \includegraphics[width=\textwidth]{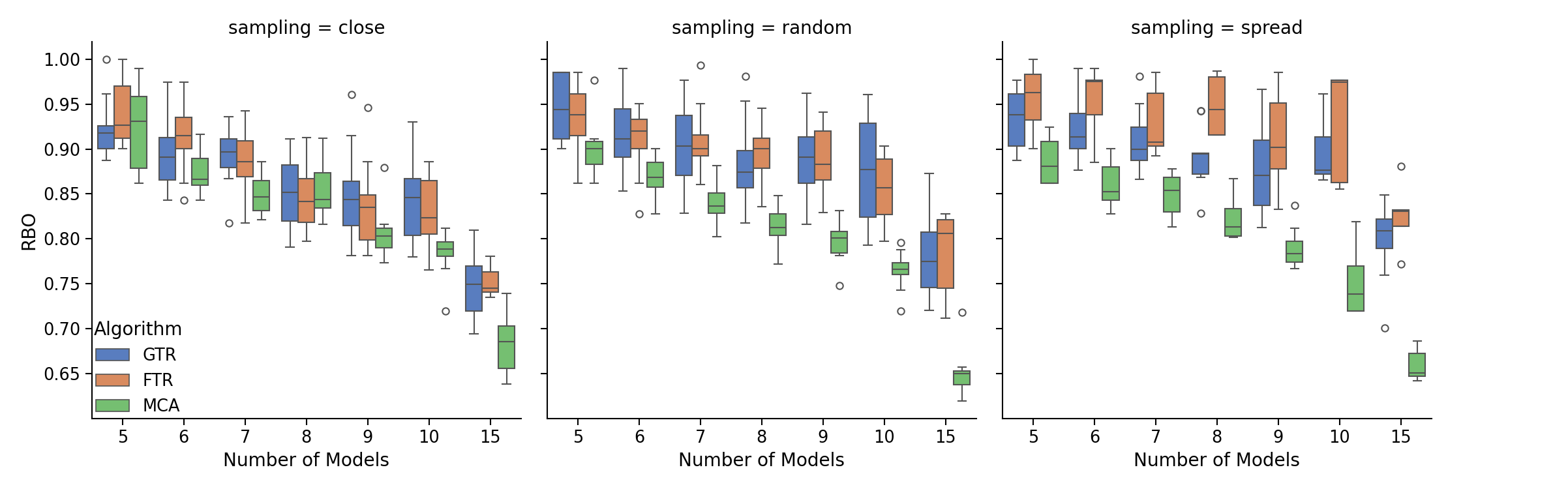}
    \end{subfigure}
    \begin{subfigure}[b]{\textwidth}
        \centering
        \includegraphics[width=\textwidth]{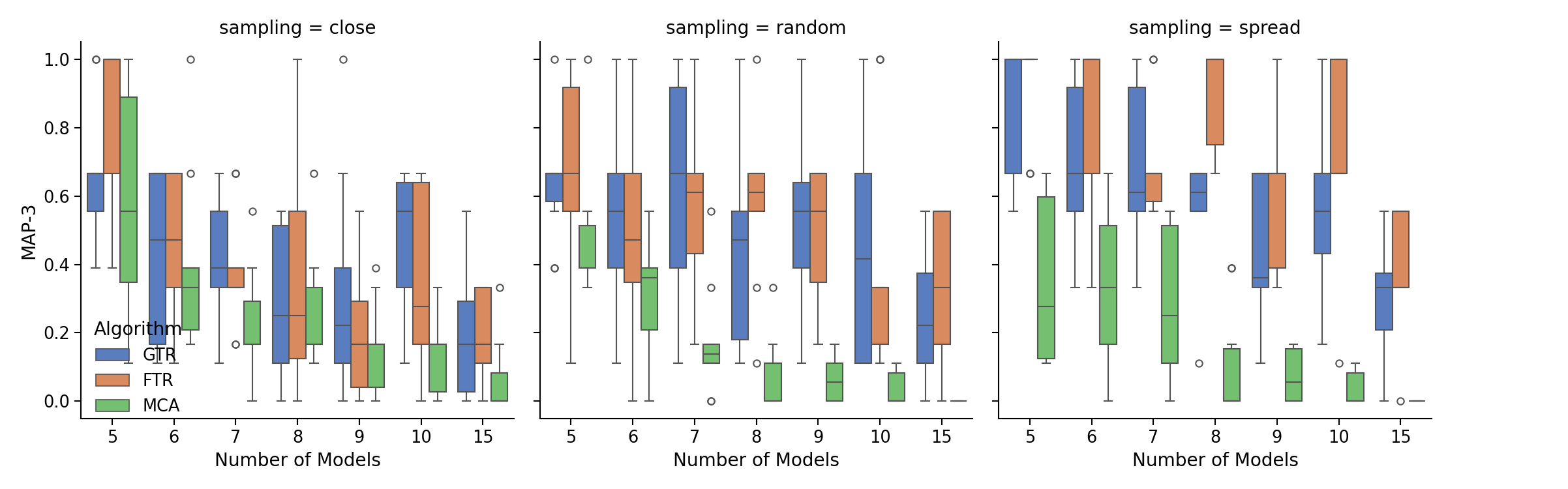}
    \end{subfigure}
    \begin{subfigure}[b]{\textwidth}
            \centering
            \includegraphics[width=\textwidth]{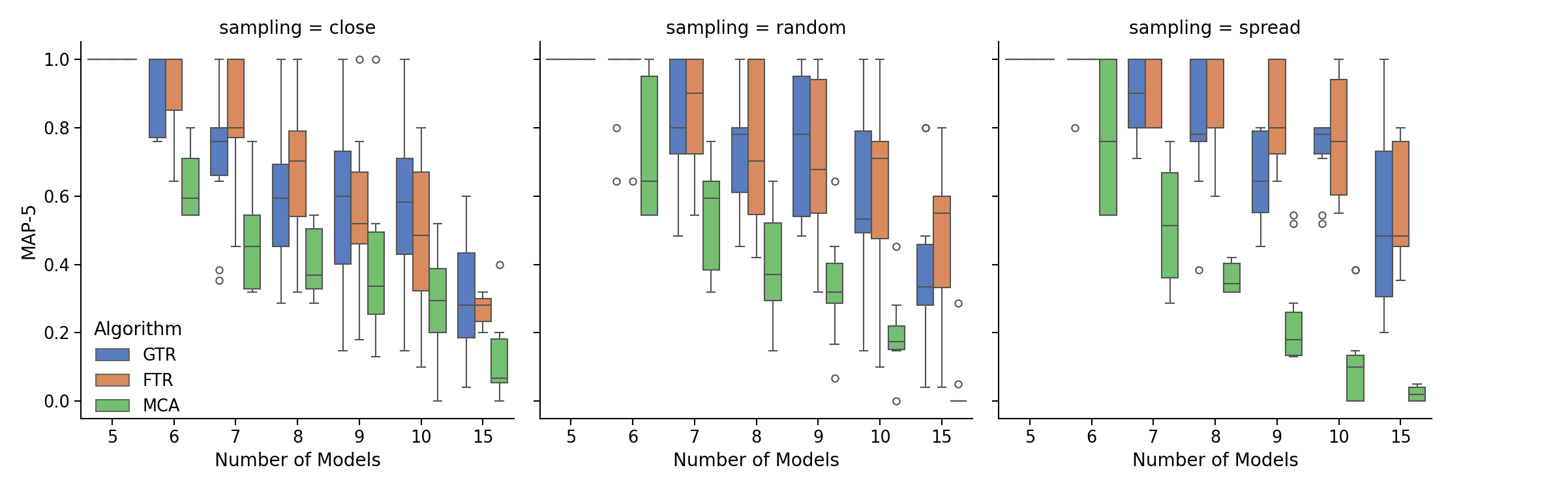}
    \end{subfigure}
       \caption{Summarization - CNN/DM: Evaluation measures as a function of sampling. RBO (top), MAP-3 (middle), MAP-5 (bottom).}
       \label{fig:app:summ:metrics-vs-sampling}
\end{figure*}

\begin{figure*}[ht]
    \centering
    \begin{subfigure}[b]{\textwidth}
        \centering
        \includegraphics[width=\textwidth]{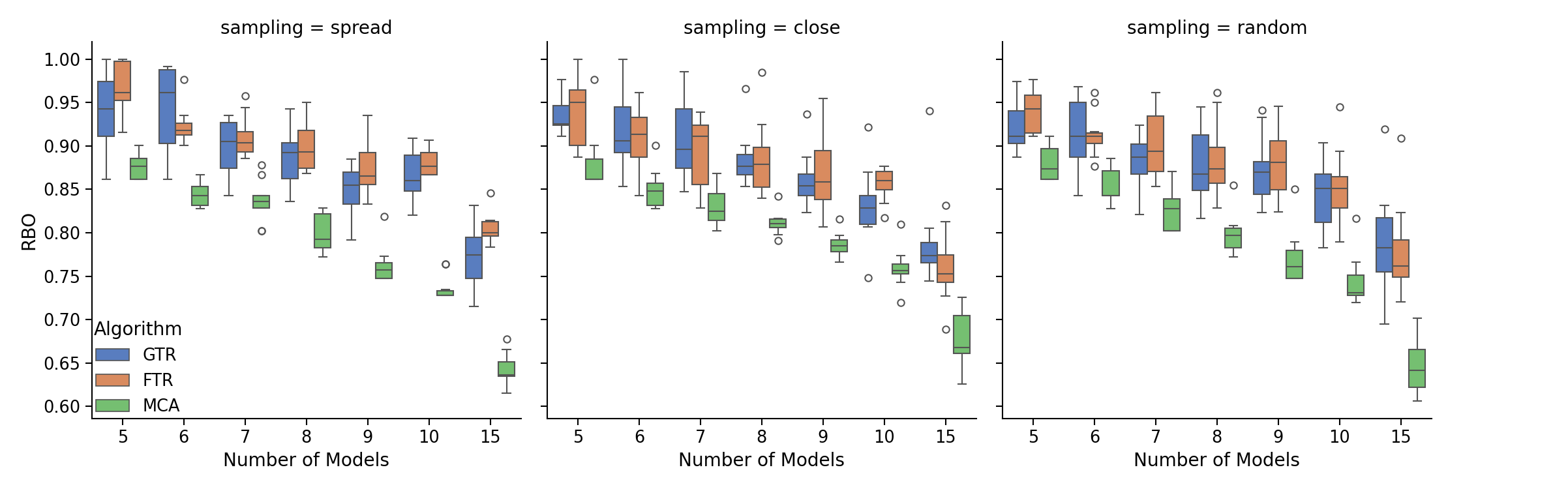}
    \end{subfigure}
    \begin{subfigure}[b]{\textwidth}
        \centering
        \includegraphics[width=\textwidth]{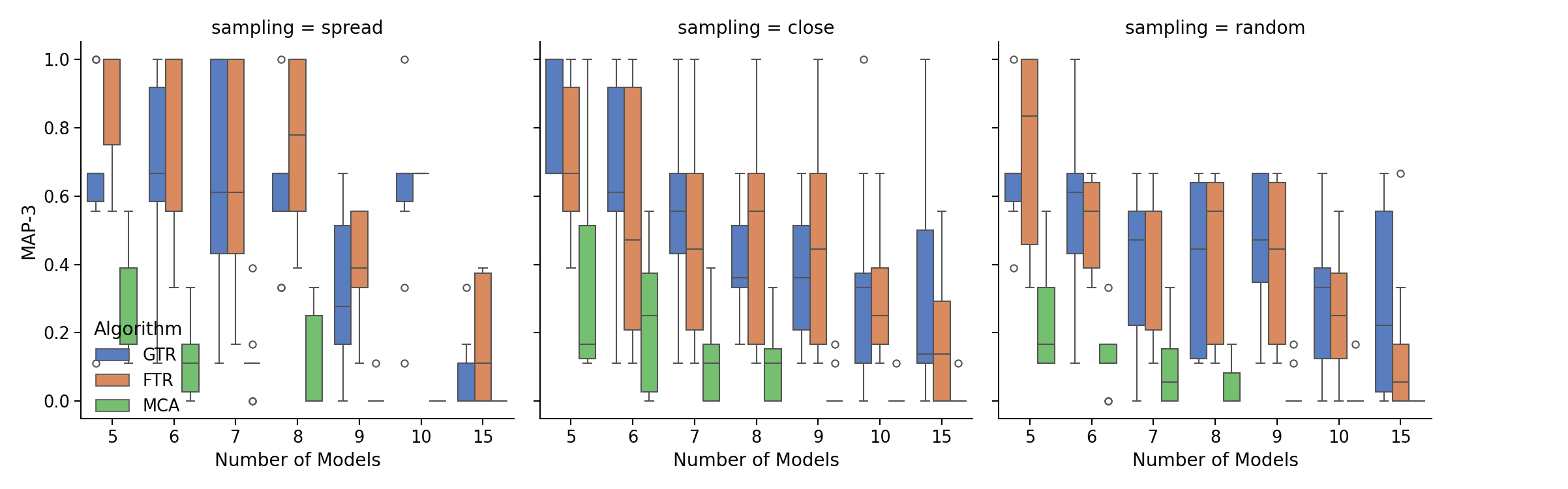}
    \end{subfigure}
    \begin{subfigure}[b]{\textwidth}
            \centering
            \includegraphics[width=\textwidth]{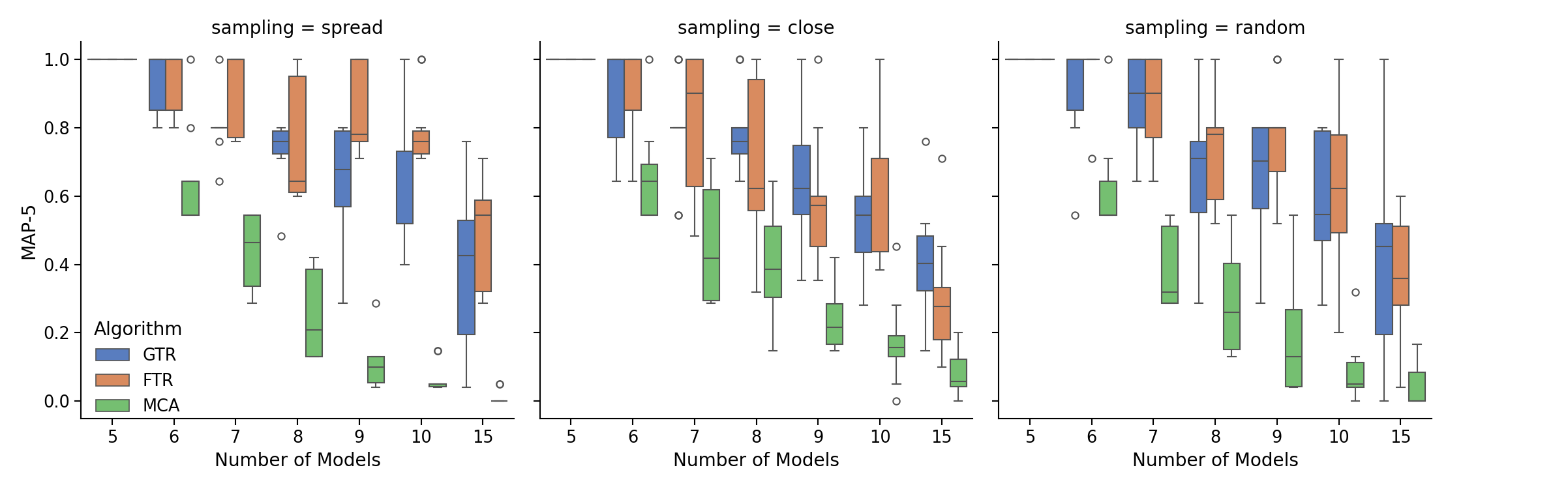}
    \end{subfigure}
       \caption{Summarization - XSUM: Evaluation measures as a function of sampling. RBO (top), MAP-3 (middle), MAP-5 (bottom).}
       \label{fig:app:summ:metrics-vs-sampling:xsum}
\end{figure*}

Figure \ref{fig:app:summ:bump} shows sample bump charts that visualize rank changes between the `true' (based on reference data) and the estimated ranks from the proposed methods. Figure \ref{fig:bump:a} shows perfect agreement between estimated and true rankings. In Figure \ref{fig:bump:b} rankings are not exactly recovered, but the clusters of good and bad models are persisted, even if the exact rankings of them within the clusters are not correct. This is also seen in Figure \ref{fig:bump:c} and \ref{fig:bump:d}.

\begin{figure*}
     \centering
     \begin{subfigure}[b]{0.24\textwidth}
         \centering
         \includegraphics[width=\textwidth]{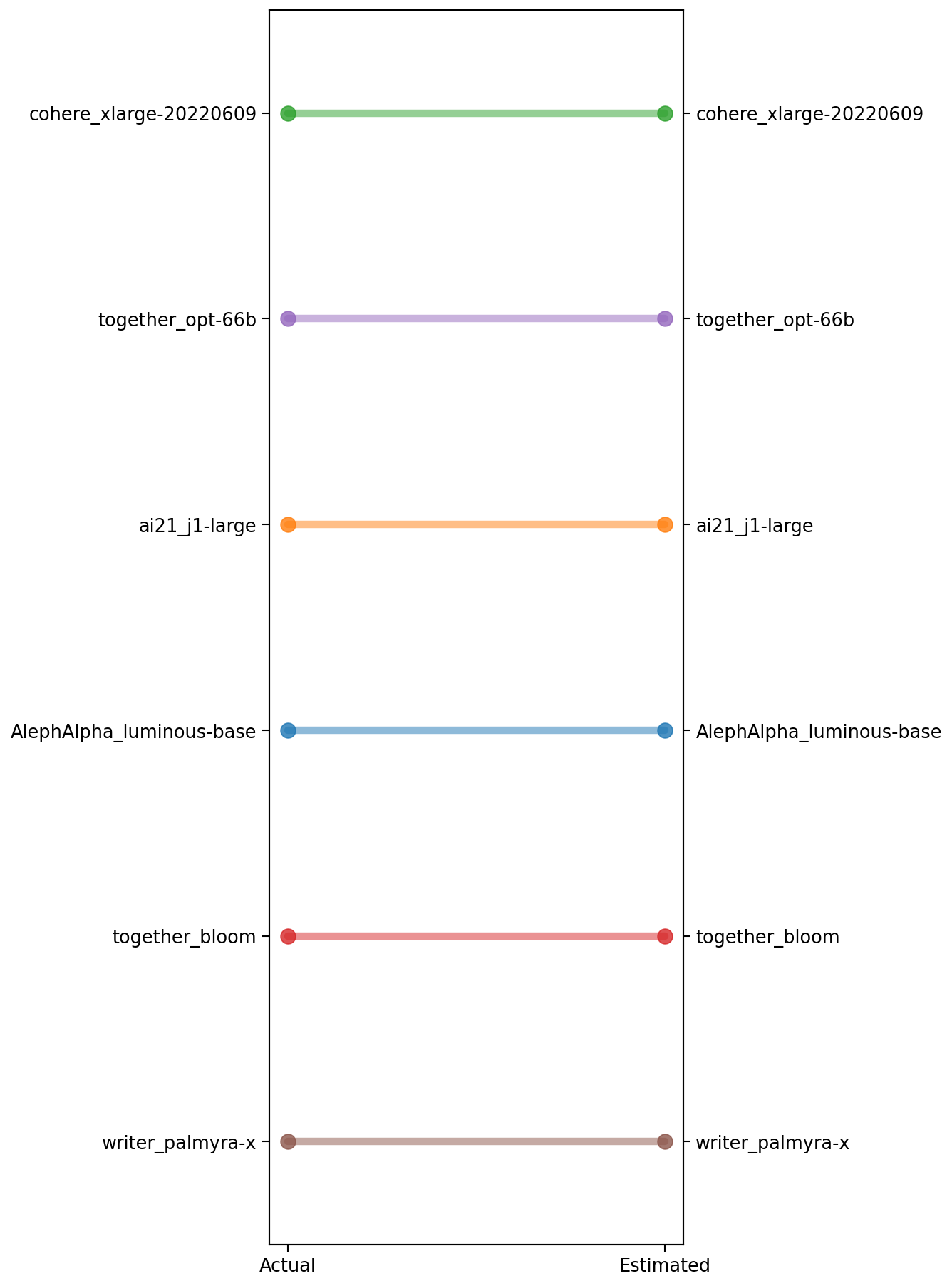}
         \caption{GTR, 6 models, close samples}
         \label{fig:bump:a}
     \end{subfigure}
     \begin{subfigure}[b]{0.24\textwidth}
         \centering
         \includegraphics[width=\textwidth]{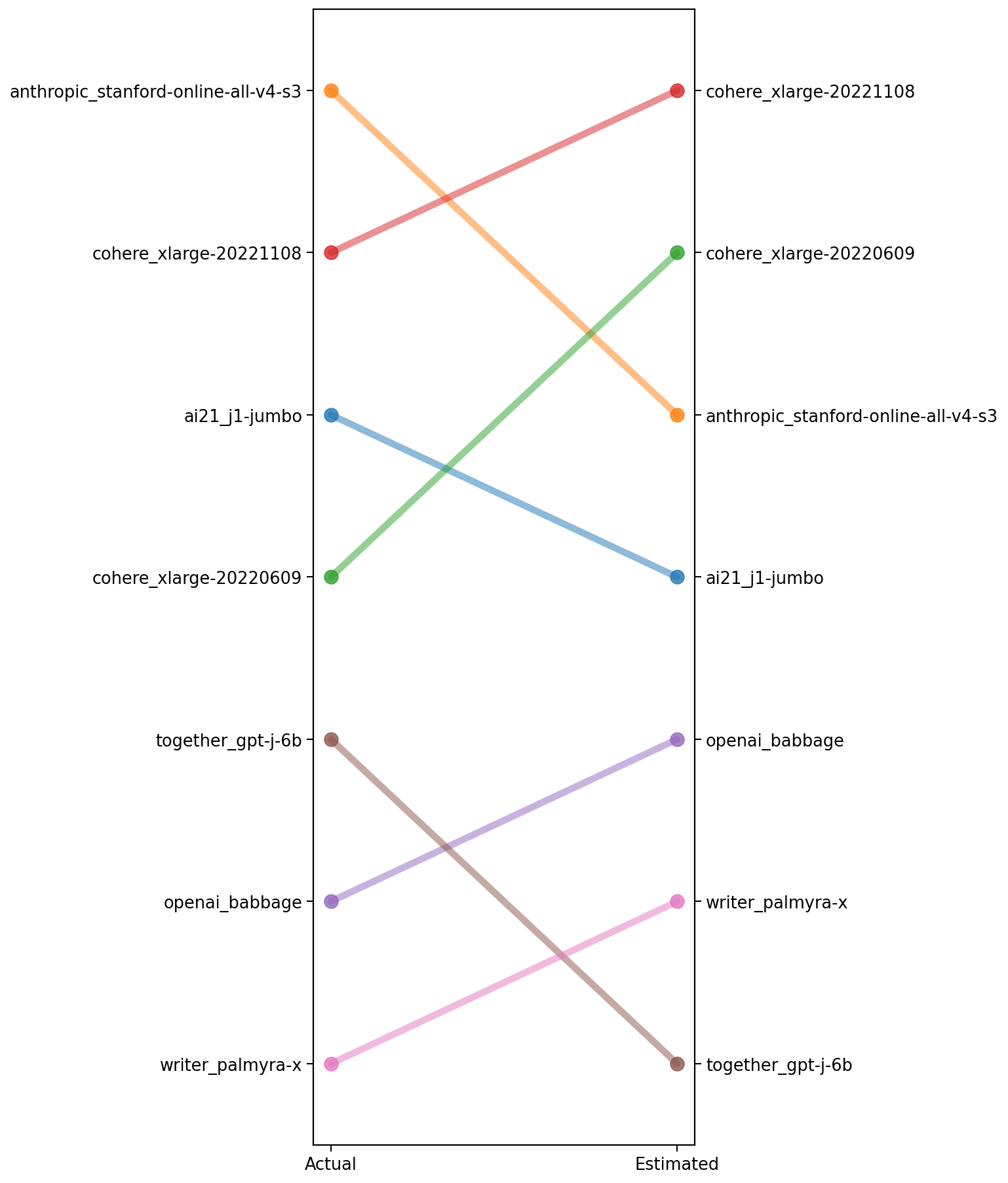}
         \caption{GTR, 7 models, spread samples}
         \label{fig:bump:b}
     \end{subfigure}
     \begin{subfigure}[b]{0.24\textwidth}
         \centering
         \includegraphics[width=\textwidth]{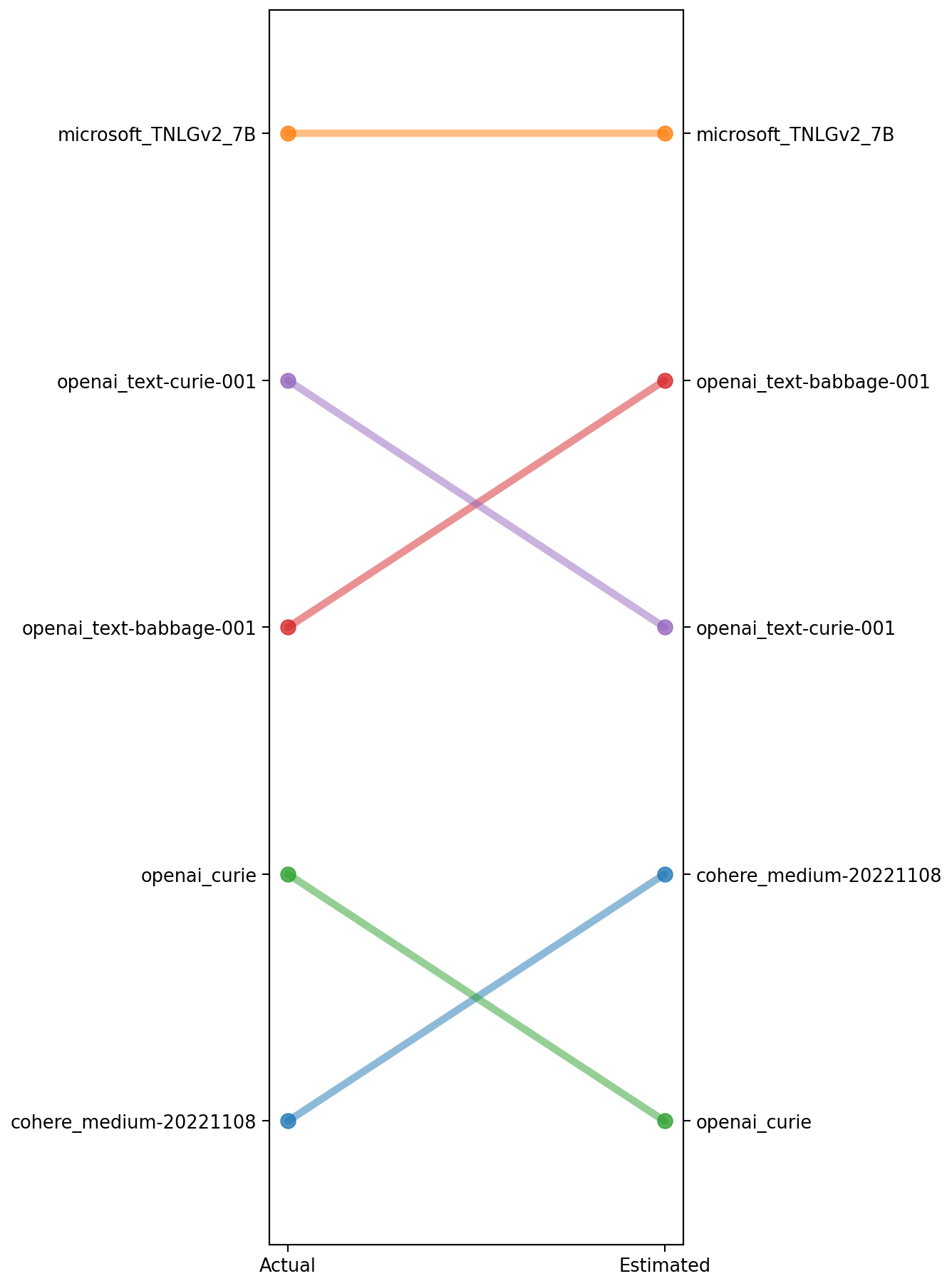}
         \caption{FTR, 5 models, close samples}
         \label{fig:bump:c}
     \end{subfigure}
     \begin{subfigure}[b]{0.24\textwidth}
         \centering
         \includegraphics[width=\textwidth]{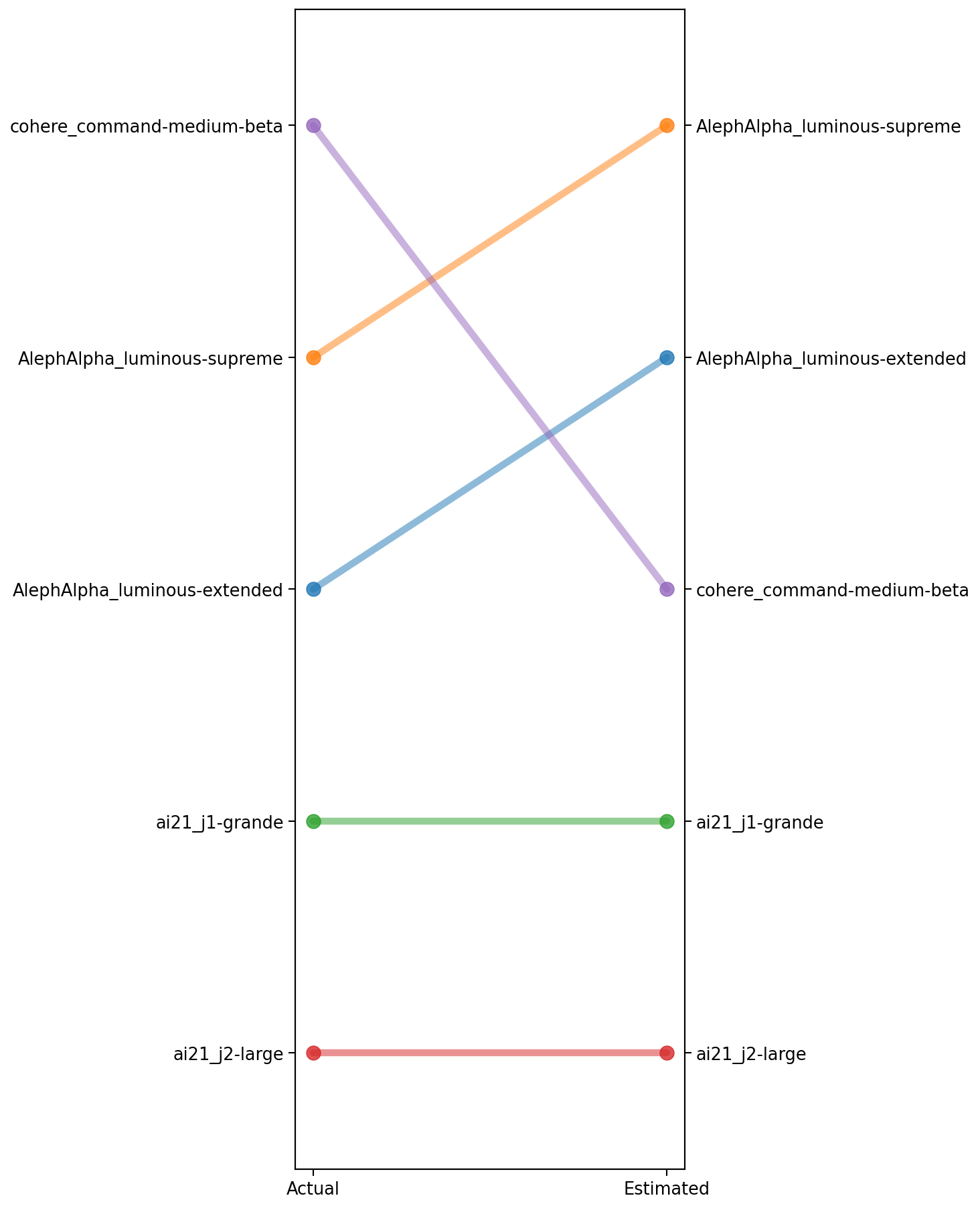}
         \caption{GTR, 5 models, close samples}
         \label{fig:bump:d}
     \end{subfigure}
        \caption{Sample bump charts showing how rankings estimated by our methods (right) compare to those derived using reference data (left). Best ranks on top.}
        \label{fig:app:summ:bump}
\end{figure*}

\subsection{Multiple Choice}
\label{subsec:multiplechoice}

For simulating this data, we generate a matrix of responses, with each row denoting a questions and each column denoting a model (of known accuracy). Each answer takes on a discrete value depending on the number of possible answers. We consider different sizes of prompts, i.e. $[100, 500]$, number of possible answers ranging from $2$ to $50$, and models of varying accuracy. The accuracy is set to be between $90$ and $10$ and experiments are run with various combinations to simulate model performance differences. As evaluation function, $g$ in FTR and $f$ in GTR, we consider the equality operator. In addition we also experimented with a noisy equality operator, that randomly flipped the outcome of equality.

The intent of the experiments was to stress test our ranking methods. Figure \ref{fig:app:synth:rbo-vs-nmodels} plots the full set of results over all combinations when the evaluations are based on equality. Generally, the ensembling method of MCA performs well with FTR showing comparable performance. See also Figure \ref{fig:app:synth:map5-vs-nmodels} for similar trends in the MAP-5 metric. 

When the number of possible answers is low, the triplet evaluations face difficulty in disambiguation. For all methods, when the possible outcomes are $2$ for example, RBO values are low. Weak models in these cases unite to promote the wrong answers.

Generally, rankings can be recovered reasonably well (provided there are sufficient possible answers) even if the best model performance is at $50\%$. 

\begin{figure*}
    \centering
    \includegraphics[width=\textwidth]{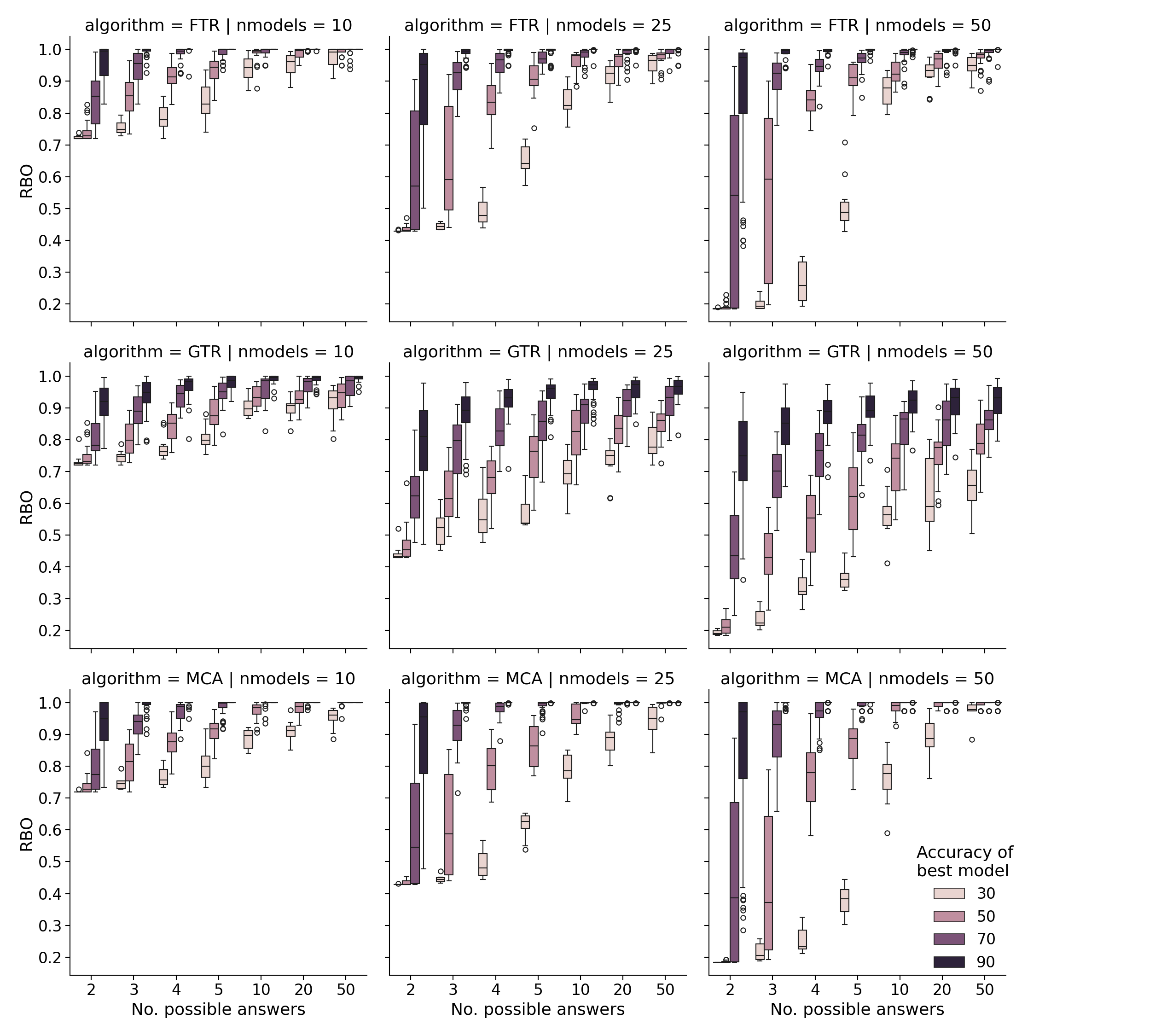}
    \caption{Multiple choice: Full version of Figure \ref{fig:exp:synth:metrics} showing summary of RBO for all runs.}
    \label{fig:app:synth:rbo-vs-nmodels}
\end{figure*}

\begin{figure*}
    \centering
    \includegraphics[width=\textwidth]{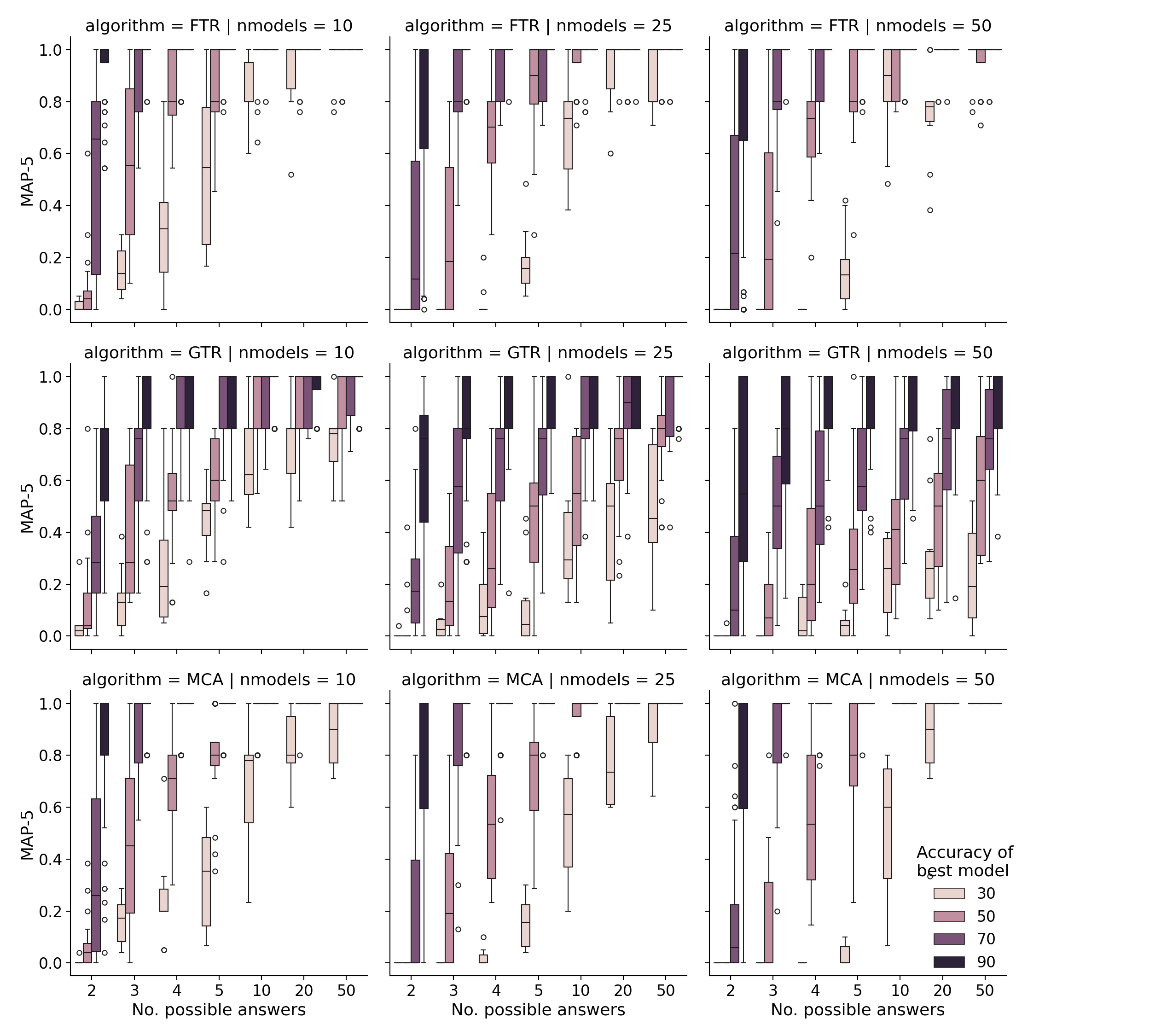}
    \caption{Multiple choice: MAP-5 for all for all runs.}
    \label{fig:app:synth:map5-vs-nmodels}
\end{figure*}

On noisy equality experiments, where the triplet evaluations are randomly perturbed with a known probability, we see that efficacy of our methods decline with increasing noise in Figure \ref{fig:app:synth:noise}. FTR and even GTR are more robust than MCA at low-medium levels of noise. This is shown for some specific sizes of instances, the patterns are similar for other cases as well.

\begin{figure*}
     \centering
     \begin{subfigure}[b]{0.9\textwidth}
         \centering
         \includegraphics[width=\textwidth]{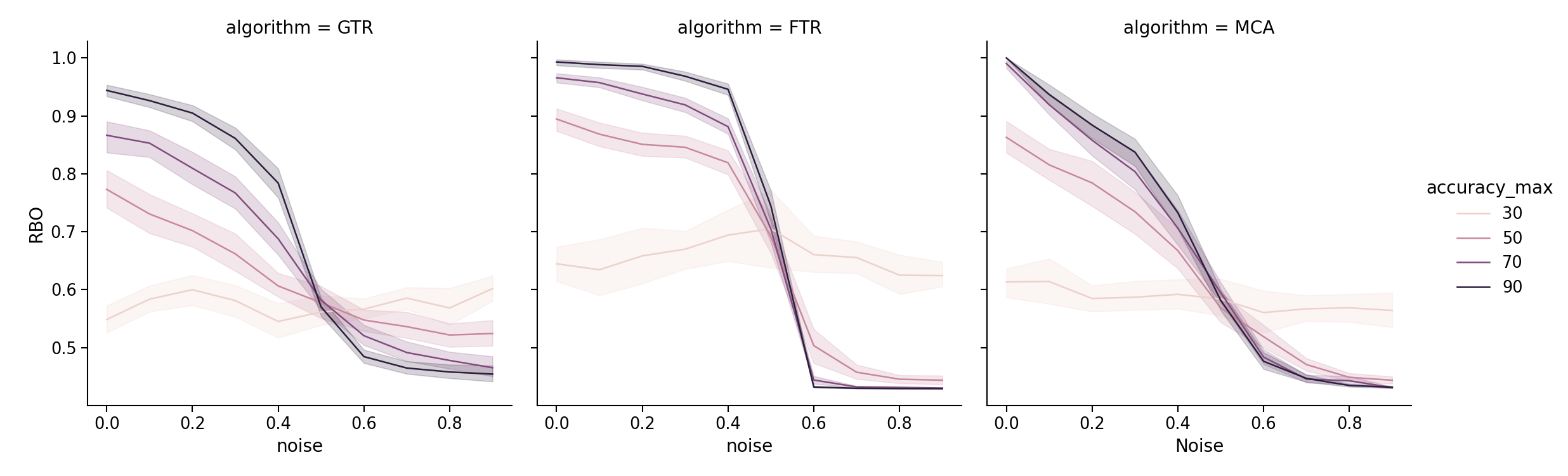}
         \caption{25 models, 5 questions}
         \label{fig:app:synth:noise:1}
     \end{subfigure}
     \begin{subfigure}[b]{0.9\textwidth}
         \centering
         \includegraphics[width=\textwidth]{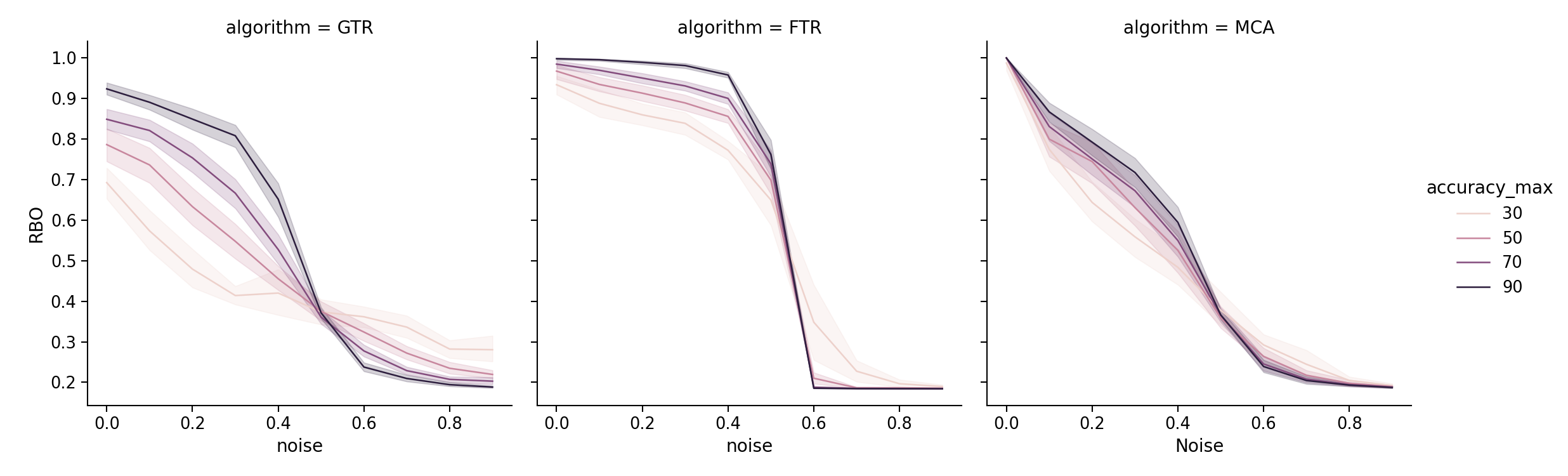}
         \caption{50 models, 50 questions}
         \label{fig:app:synth:noise:2}
     \end{subfigure}
        \begin{subfigure}[b]{0.9\textwidth}
         \centering
         \includegraphics[width=\textwidth]{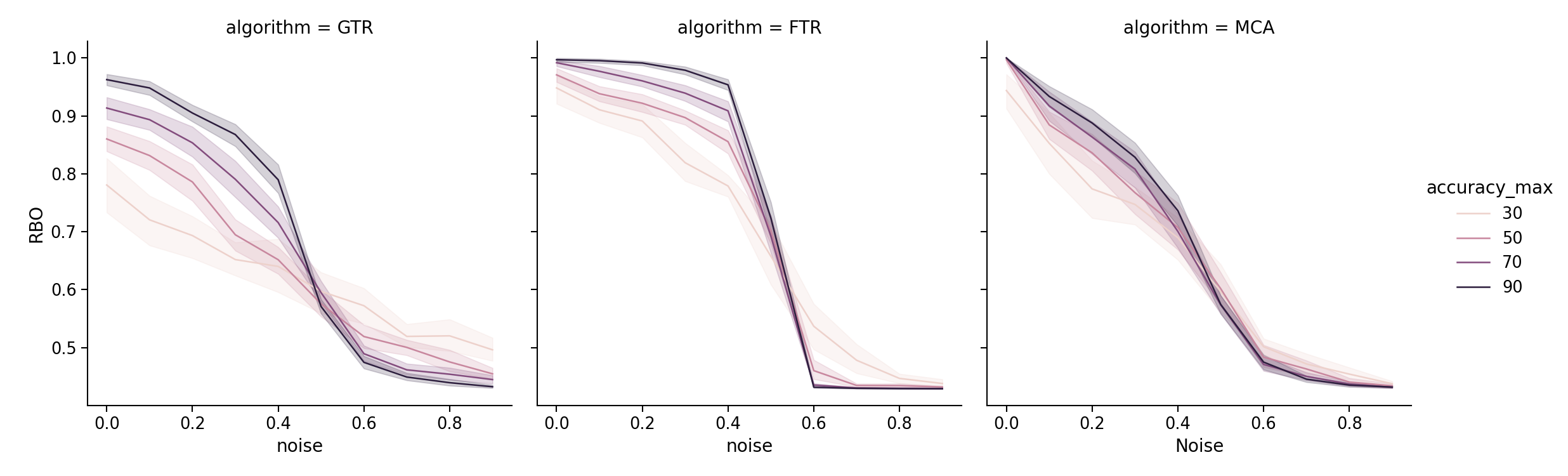}
         \caption{25 models, 50 questions}
         \label{fig:app:synth:noise:3}
     \end{subfigure}
          \begin{subfigure}[b]{0.9\textwidth}
         \centering
         \includegraphics[width=\textwidth]{figures/synth-rbo-vs-noise-10-50.png}
         \caption{10 models, 50 questions}
         \label{fig:app:synth:noise:4}
     \end{subfigure}
        \caption{Impact of noisy evaluations on multiple choice data for selected cases.}
        \label{fig:app:synth:noise}
\end{figure*}

\subsection{Models used}
\label{app:models}
Table \ref{tab:app:models} shows the list of models used in the HELM experiments for summarization. 

\begin{table}[]
    \centering
\begin{tabular}{rl}
\toprule
No. & Model \\
\midrule
1 & AlephAlpha\_luminous-base \\
2 & AlephAlpha\_luminous-extended \\
3 & AlephAlpha\_luminous-supreme \\
4 & ai21\_j1-grande \\
5 & ai21\_j1-grande-v2-beta \\
6 & ai21\_j1-jumbo \\
7 & ai21\_j1-large \\
8 & ai21\_j2-grande \\
9 & ai21\_j2-jumbo \\
10 & ai21\_j2-large \\
11 & anthropic\_stanford-online-all-v4-s3 \\
12 & cohere\_command-medium-beta \\
13 & cohere\_command-xlarge-beta \\
14 & cohere\_large-20220720 \\
15 & cohere\_medium-20220720 \\
16 & cohere\_medium-20221108 \\
17 & cohere\_small-20220720 \\
18 & cohere\_xlarge-20220609 \\
19 & cohere\_xlarge-20221108 \\
20 & microsoft\_TNLGv2\_530B \\
21 & microsoft\_TNLGv2\_7B \\
22 & openai\_ada \\
23 & openai\_babbage \\
24 & openai\_curie \\
25 & openai\_davinci \\
26 & openai\_gpt-3.5-turbo-0301 \\
27 & openai\_text-ada-001 \\
28 & openai\_text-babbage-001 \\
29 & openai\_text-curie-001 \\
30 & openai\_text-davinci-002 \\
31 & openai\_text-davinci-003 \\
32 & together\_bloom \\
33 & together\_gpt-j-6b \\
34 & together\_gpt-neox-20b \\
35 & together\_opt-175b \\
36 & together\_opt-66b \\
37 & together\_redpajama-incite-base-3b-v1 \\
38 & together\_yalm \\
39 & writer\_palmyra-instruct-30 \\
40 & writer\_palmyra-x \\
\bottomrule
\end{tabular}
    \caption{Models used in summarization experiments}
    \label{tab:app:models}
\end{table}

\subsection{Synthetic Data Generation}
\label{subsec:app:synth}

Table \ref{tab:app:synth} shows an example of synthetic data generated for the multiple choice setting. We parameterize the generation on model accuracy, number of possible answers, and the number of questions in the dataset. 

\begin{table}[]
\begin{tabular}{lrrrrrr}
\toprule
 & GT & M1 & M2 & M3 & M4 & M5 \\
\midrule
Q0 & 0 & 0 & 0 & 0 & 7 & 0 \\
Q1 & 2 & 6 & 2 & 2 & 2 & 1 \\
Q2 & 1 & 1 & 1 & 1 & 0 & 7 \\
Q3 & 6 & 6 & 6 & 4 & 6 & 6 \\
Q4 & 5 & 5 & 1 & 5 & 5 & 5 \\
\bottomrule
\end{tabular}
    \caption{Sample generated dataset with model accuracies in the range $50-80\%$, with $10$ possible answers per question. GT denotes the ground truth, i.e. true label, and M{1-5} denote simulated model responses.}
    \label{tab:app:synth}
\end{table}

\section{Most Common Answer}
\label{subsec:app:mca}

We now present additional details on the Most-Common Answer (MCA) method that we use as a baseline. When the context is multiple choice, the concept of the most common answer is straight forward. You assign the label with the majority vote from all the LLMs being tested as the true label. For instance, in a multiple choice setting with $4$ options $A, B, C, D$, if from five different models you gather responses for a specific question as $A, A, B, C, D$, then $A$ would be considered as the true label. 

In generative settings, such as for summarization, what constitutes a `common' answer involves more work. To establish an MCA baseline, we first generate all n-grams from all model responses. We then use the most frequent n-grams across all responses upto a specific size, and consider that to be the `true' response. Specific answers from the models are then evaluated against this to generate the closest response. 

This is illustrated using a simple example. Consider three models' responses to the question \emph{What is the capital of Canada?} in Table \ref{tab:app:canada}. From all the responses, we compute frequencies of all bigrams to get \texttt{'ta': 3, 'nt': 2, 'Ot': 2, 'tt': 2, 'aw': 2, 'wa': 2,'To': 1,'or': 1,'ro': 1,'on': 1,'to': 1,'a,': 1,', ': 1,' O': 1,'On': 1,'ar': 1,'ri': 1,'io': 1}. From this, we select the top-k bigrams to constitute the most-common answer. In this example we assume the top five, but in the summarization experiments we consider k to be 256. We can now compute similarity measures of each model response relative to this most-common answer. Rouge-2 scores, as an example are shown in Table \ref{tab:app:canada} indicating that M3 has the best response.

\begin{table}[hb]
    \centering
\begin{tabular}{llr}
\toprule
 Model & Response & Rouge-2 w.r.t. MCA \\
\midrule
M1 & Toronto & 0.117\\
M2 & Ottawa, Ontario & 0.480\\
M3 & Ottawa & \textbf{0.500}\\ \hline
MCA & ta: 3, nt: 2, ot: 2, & \\
     &   tt: 2, aw: 2 & - \\
\bottomrule
\end{tabular}
    \caption{Illustration of model responses and most-common answer (MCA) for the question: `What is the capital of Canada?'}
    \label{tab:app:canada}
\end{table}

\section{LLM as a Judge}
\label{subsec:app:prometheus}

We prompt Prometheus-2 \cite{kim2024prometheus} to compare the results of two summaries from the CNN/DM summarization dataset. This is done for all pairs of LLMs and for a sample of $50$ questions in the benchmark. For each evaluation, the LLM Judge declares a winning LLM. Our task criteria is simply `Which is the better response?'. Generally, human annotators judging summarization quality do so along specific attributes such as coherence, fluency, consistency, relevance, and \cite{fabbri2021summeval}. We additionally compared Prometheus with human annotations along these dimensions and find poor agreement. For this task, the LLM Judge is overly pessimistic across all dimensions, see Figure \ref{fig:app:prom}.
\begin{figure*}
    \centering
    \includegraphics[width=\textwidth]{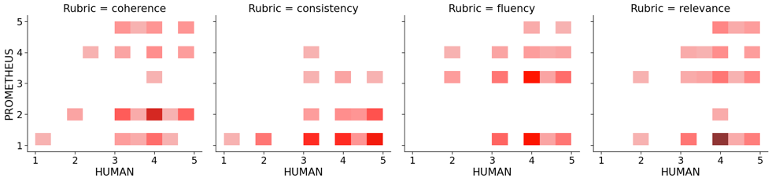}
    \caption{Comparison of Human annotations with LLM Judge (Prometheus) annotations based on the SummEval dataset \cite{fabbri2021summeval}.}
    \label{fig:app:prom}
\end{figure*}

Next we considered pairwise evaluations, specifically, we prompted Prometheus-2 with the following prompt template:
\small
\begin{verbatim}
"""###Task Description:
An instruction (might include an Input inside it), 
a response to evaluate, and a score rubric 
representing  a evaluation criteria are given.
1. Write a detailed feedback that assess the 
quality of two responses strictly based on the 
given score rubric, not evaluating in general.
2. After writing a feedback, choose a better 
response between Response 1 and Response 2. 
You should refer to the score rubric.
3. The output format should look as follows: 
"Feedback: (write a feedback for criteria) 
[RESULT] (1 or 2)"
4. Please do not generate any other opening, 
closing, and explanations.
###Instruction:
{instruction}
###Response 1:
{response_1}
###Response 2:
{response_2}
###Score Rubric:
{rubric}
###Feedback: 
"""
\end{verbatim}
\normalsize

No additional checks for positional bias were performed. 